%% file: arxiv.tex
\documentclass{article}
\usepackage{arxiv}

\usepackage[utf8]{inputenc} 
\usepackage[T1]{fontenc}    
\usepackage{hyperref}       
\usepackage{url}            
\usepackage{booktabs}       
\usepackage{amsfonts}       
\usepackage{nicefrac}       
\usepackage{microtype}      
\usepackage{lipsum}		
\usepackage{graphicx}
\usepackage{natbib}
\usepackage{doi}

\title{Federated Learning Can Find Friends That Are Advantageous}

\date{} 					

\author{
    \hspace{1mm}Nazarii Tupitsa
    \\
	Department of Machine Learning\\
	MBZUAI\\
	\And
    \hspace{1mm}Samuel Horv\'{a}th 
    \\
	Department of Machine Learning\\
	MBZUAI\\
	\And
    \hspace{1mm}Martin Tak\'{a}\v{c}
    \\
	Department of Machine Learning\\
	MBZUAI\\
	\And
    \hspace{1mm}Eduard Gorbunov
    \\
	Department of Machine Learning\\
	MBZUAI\\
}

\renewcommand{\shorttitle}{Federated Learning Can Find Friends That Are Advantageous}

\hypersetup{
pdftitle={Federated Learning Can Find Friends That Are Advantageous},
pdfauthor={N. Tupitsa, S. Horvath, M. Takac, E. Gorbunov},
}

\usepackage{microtype}
\usepackage{graphicx}
\usepackage{subfigure}
\usepackage{booktabs} 

\usepackage{hyperref}

\usepackage{algorithm}
\usepackage{algorithmicx}
\usepackage{algpseudocode}

\usepackage{amsmath}
\usepackage{amssymb}
\usepackage{mathtools}
\usepackage{amsthm}

\usepackage[capitalize,noabbrev]{cleveref}

\theoremstyle{plain}
\newtheorem{theorem}{Theorem}[section]

\theoremstyle{definition}

\newtheorem{assumption}[theorem]{Assumption}
\theoremstyle{remark}

\usepackage{enumitem}

\usepackage[textsize=tiny]{todonotes}

\usepackage{hyperref}
\usepackage{multirow}
\usepackage{colortbl}
\definecolor{niceblue}{rgb}{0.0,0.19,0.56}
\usepackage{url}
\input{defs}

\newcommand{\f}{f}
\newcommand{\xiv}{\bm{\xi}}

\newcommand{\g}{g}
\newcommand{\x}{x}
\newcommand{\y}{y}
\newcommand{\w}{w}
\newcommand{\n}{{n}}
\newcommand{\sumin}{\sum_{i=0}^{\n-1}}
\newcommand{\gn}[1][]{G_{#1}}
\usepackage{xspace}
\newcommand{\algname}[1]{{\texttt{#1}}\xspace}

\usepackage{enumitem}

\newcommand{\answerYes}[1][]{\textcolor{blue}{[Yes] #1}}

\newcommand{\answerNA}[1][]{\textcolor{gray}{[NA] #1}}

\newcommand{\Algn}{\algname{MeritFed}}

\begin{document}

\maketitle

\begin{abstract}
In Federated Learning (FL), the distributed nature and heterogeneity of client data present both opportunities and challenges. While collaboration among clients can significantly enhance the learning process, not all collaborations are beneficial; some may even be detrimental. In this study, we introduce a novel algorithm that assigns adaptive aggregation weights to clients participating in FL training, identifying those with data distributions most conducive to a specific learning objective. We demonstrate that our aggregation method converges no worse than the method that aggregates only the updates received from clients with the same data distribution. Furthermore, empirical evaluations consistently reveal that collaborations guided by our algorithm outperform traditional FL approaches. This underscores the critical role of judicious client selection and lays the foundation for more streamlined and effective FL implementations in the coming years.
\end{abstract}
\section{Introduction}
Federated Learning (FL) introduces an innovative paradigm redefining traditional machine learning workflow. Instead of centrally pooling sensitive client data, FL allows for model training on decentralized data sources stored directly on client devices
\citet{konevcny2016federated,zhang2021survey,li2020federated,beznosikov2021decentralized}. In this approach, rather than training Machine Learning (ML) models in a centralized manner, a shared model is distributed to all clients. Each client then performs local training, and model updates are exchanged between clients and the FL orchestrator (often referred to as the master server) \citet{mcmahan2017communication,shokri2015privacy,karimireddy2020scaffold}.
\\[5pt]
{\bf Personalized Federated Learning (PFL).}
The concept of PFL \citet{collins2021exploiting, hanzely2020lower, sadiev2022decentralized,AlmansooriColl} has been gaining traction. In this framework, each client, often referred to as an agent, takes part in developing their own personalized model variant. This tailored training approach leverages local data distributions, aiming to design models that cater to the distinct attributes of each client's dataset \citet{fallah2020personalized}.
However, a prominent challenge arises in this decentralized training landscape due to the data's non-IID (independent and identically distributed) nature across various clients. Data distributions that differ considerably can have a pronounced impact on the convergence and generalization capabilities of the trained models.
While certain client-specific data distributions might strengthen model performance, others could prove detrimental, introducing biases or potential adversarial patterns.
Additionally, within the personalized federated learning paradigm, the emphasis on crafting individualized models could inadvertently heighten these data disparities
\citet{kairouz2021advances}. Consequently, this may lead to models that deliver subpar or, in some cases, incorrect results when applied to wider or diverse datasets \citet{kulkarni2020survey}.
\\[5pt]
{\bf Collaboration as a service.} In this paper, we introduce a modified protocol for FL that deviates from a strictly personalized approach. Rather than focusing solely on refining individualized models, our approach seeks to harness the advantages of distinct data distributions, curb the detrimental effects of outlier clients, and promote collaborative learning. Through this innovative training mechanism, our algorithm discerns which clients are optimal collaborators to ensure faster convergence and potentially better generalization.


\subsection{Setup}
We assume that there are $n$ clients participating in the training and consider the first one as a target client. The goal is to train the model for this client, i.e., we consider
\begin{equation}
    \textstyle{\min}_{x\in\R^d} \left\{f(x) \equiv f_1(x) := \E_{\xi_1 \sim \cD_1} [f_{\xi_1}(x)]\right\}, \label{eq:main_problem}
\end{equation}
where $f_{\xi_1}:\R^d \to \R$ is the loss function on sample $\xi_1$ and $f:\R^d \to \R$ is an expected loss. Other clients can also have data sampled from similar distributions, but we also allow adversarial participants, e.g., Byzantines~\citet{lamport1982byzantine, lyu2020privacy}. That is, some clients can be beneficial for the training in certain stages, but they are not assumed to be known apriori.

The considered target client scenario naturally arises in FL on medical image data. In such applications, different hospitals naturally have different data distributions (e.g., due to the differences in the equipment). Therefore, the data coming from one clinic can be useless to another clinic. At the same time, several clinics can have similar data distributions.

\subsection{Contribution}

Our main contributions are listed below.
\begin{itemize}[noitemsep,topsep=0pt,left=0pt]
    \item \textbf{New method: \Algn.} We proposed a new method called Merit-based Federated Averaging for Diverse Datasets (\Algn) that aims to solve \eqref{eq:main_problem}. The key idea is to use the stochastic gradients received from the clients to adjust the weights of averaging through the inexact solving of the auxiliary problem of minimizing a validation loss as a function of aggregation weights. 

    \item \textbf{Provable convergence under mild assumptions.} We prove that \Algn converges not worse than \algname{SGD} that averages only the stochastic gradients received from clients having the same data distribution (these clients are not known apriori) for smooth non-convex and Polyak-Lojasiewicz functions under standard bounded variance assumption.

    \item \textbf{Utilizing all possible benefits.} We numerically show that \Algn can even benefit from collaboration with clients having different data distributions when these distributions are close to the target one. That is, \algname{MeritFed} automatically detects beneficial clients at any stage of training. Moreover, we illustrate the Byzantine robustness of the proposed method even when Byzantine workers form a majority.
\end{itemize}





\subsection{Related work}

{\bf Federated optimization.} Standard results in distributed/federated optimization focus on the  problem:
\begin{equation}
    \textstyle{\min}_{x\in \R^d}\frac{1}{n}\sum_{i=1}^n f_i(x), \label{eq:standard_ERM}
\end{equation}
where $f_i(x)$ represents either expected or empirical loss on the client $i$. This problem significantly differs from \eqref{eq:main_problem}, since one cannot completely ignore the updates from some clients to achieve a better solution. Typically, in this case, communication is the main bottleneck of the methods for solving such problems. To address this issue one can use communication compression \citet{alistarh2017qsgd, stich2018sparsified, mishchenko2019distributed}, local steps \citet{stich2018local, khaled2020tighter, kairouz2021advances, wang2021field, mishchenko2022proxskip,sadiev2022decentralized}, client importance sampling \citet{cho2020client,nguyen2020fast,ribero2020communication,lai2021oort,luo2022tackling,chen2022optimal}, or decentralized protocols \citet{lian2017can, song2022communication}. However, these techniques are orthogonal to what we focus on in our paper, though incorporating them into our algorithm is a prominent direction for future research.

{\bf Clustered FL.} Another way of utilizing benefits from the other clients is the clustering of clients based on some information about their data or personalized models. \citet{tang2021personalized} propose a personalized formulation with $\ell_2$-regularization that attracts a personalized model of a worker to the center of the cluster that this worker belongs to. A similar objective is studied by \citet{ma2022convergence}. \citet{ghosh2020efficient} develop an algorithm that updates clusters's centers using the gradients of those clients that have the smallest loss functions at the considered cluster's center. It is worth mentioning that, in contrast to our work, the mentioned works modify the personalized objective to illustrate some benefits of collaboration while we focus on the pure personalized problem of the target client. Under the assumption that the data distributions of each client are mixtures of some finite set of underlying distributions, \citet{marfoq2021federated} derive the convergence result for the Federated Expectation-Maximization algorithm. This is the closest work to our setup in the Clustered FL literature. However, in contrast to \citet{marfoq2021federated}, we do not assume that the gradients are bounded, and the local loss functions have bounded gradient dissimilarity. Another close work to ours is \citep{fraboni2021clustered}, where the authors consider so-called clustered-based sampling. However, \citet{fraboni2021clustered} also make a non-standard assumption on the bounded dissimilarity of the local loss functions, while one of the key properties of our approach is its robustness to arbitrary clients' heterogeneity. \citet{li2020federatedopt} is also a relevant paper in the sense that not all workers are selected for aggregation at each communication round (due to the client sampling). However, this work focuses on weighted empirical risk minimization (with weights proportional to the dataset size), i.e.,~\citet{li2020federatedopt} consider a different problem.

{\bf Non-uniform averaging.} There are also works studying the convergence of distributed \algname{SGD}-type methods that use non-uniform (but fixed) weights of averaging. \citet{ding2022collaborative} propose a method to detect collaboration partners and adaptively learn "several" models for numerous heterogeneous clients. Directed graph edge weights are used to calculate group partitioning. Since the calculation of optimal weights in their approach is based on similarity measures between clients' data, it is unclear how to compute them in practice without sacrificing the users's data privacy. \citet{even2022sample} develop and analyze another approach for personalized aggregation, where each client filters gradients and aggregates them using fixed weights. The optimal weights also require estimating the distance between distributions (or communicating empirical means among all clients and estimating effective dimensions). 
Both works do not consider weights evolving in time, which is one of the key features of our method.
 
Non-fixed weights are considered in~\citet{wu2021fast}, but the authors focus on non-personalized problem formulation. In particular, \citet{wu2021fast} propose the method called \algname{FedAdp} that uses cosine similarity between gradients and the Gompertz function for updating aggregation weights. Under the strong bounded local gradient dissimilarity assumption\footnote{\citet{wu2021fast} assume that there exist constants $A,B > 0$ such that $A\|\nabla f(x)\| \leq \|\nabla f_i(x)\| \leq B\|\nabla f(x)\|$ for every client $i \in [n]$ and any $x$, where $f(x) = \frac{1}{n}\sum_{i=1}^n f_i(x)$.}, \citet{wu2021fast} derive a non-conventional upper bound (for the loss function at the last iterate of their algorithm) that does not necessarily imply convergence of the method. \citet{zhang2020personalized} introduces \algname{FedFomo} that uses additional data to adjust the weights of aggregation in Federated Averaging. In this context, \algname{FedFomo} is close to \algname{MeritFed}. However, the weights selection formulas significantly differ from ours. In particular,~\citet{zhang2020personalized} do not relate the proposed weights with the minimization problem from Line~\ref{lst:line:aux_problem} of our method. In addition, there is no theoretical convergence analysis of \algname{FedFomo}.

{\bf Bi-level optimization.} 
Taking into account that we want to solve problem \eqref{eq:main_problem} using the information coming from not only the target client, it is natural to consider the following bi-level optimization (BLO) problem formulation:
\begin{eqnarray}
    \min_{\w\in \Delta^\n_1}&& \f(\x^*(\w)), \label{eq:blo} \\ \text{s.t.}&& \x^*(\w) \in \textstyle{\argmin}_{\x \in \R^d} \textstyle{\sum}_{i=1}^n \w_i \f_{i}(\x), \label{eq:LL_problem}
\end{eqnarray}
where $\Delta_1^n$ is a unit simplex in $\R^n$: $\Delta_1^n =  \{\w \in \R^n\mid \sum_{i=1}^n w_i = 1,\; w_i \geq 0 \;\forall i\in [n]\}$. The problem in \eqref{eq:blo} is usually called the upper-level problem (UL), while the problem in \eqref{eq:LL_problem} is the lower-level (LL) one. Since in our case $f(x) \equiv f_1(x)$, \eqref{eq:blo}-\eqref{eq:LL_problem} is equivalent to \eqref{eq:main_problem}. In the general case, this equivalence does not always hold and, in addition, function $f$ is allowed to depend on $w$ not only through $x^*$. All these factors make the general BLO problem hard to solve. The literature for this general class of problems is quite rich, and we cover only closely related works.

The closest works to ours are \citet{chen2021weighted}, which propose so-called Target-Aware  Weighted Training (\algname{TAWT}), and its extension to the federated setup \citet{huang2022federated}. Their analysis relies on the existence of weights $\w$, such that $\dist\rbr*{\sum_{i=1}^n\w_i\cD_i , \cD_{\text{target}}} = 0$ in terms so-called representation-based distance \citet{chen2021weighted}, which is also zero in our case, or existence of identical neighbors. However, the analysis is based on BLO's techniques and requires a hypergradient estimation, i.e., $\nabla_w f(x^*(w), w)$, which is usually hard to compute. To avoid the hypergradient calculation, \citet{chen2021weighted} also propose a heuristic based on the usage of cosine similarity between the clients' gradients, which makes the implementation of the algorithm similar to \algname{FedAdp}~\citet{wu2021fast}.
 
In fact, there are two major difficulties in estimating hypergradient. The first one is that the optimal solution $\x^*(\w)$ of the lower problem for every given $\w$ needs to be estimated. The known approaches iteratively update the lower variable $\x$ multiple times before updating $\w$, which causes high communication costs in a distributed setup. A lot of methods~\citet{ghadimi2018approximation,hong2020two,chen2021closing,ji2021bilevel,ji2022will} are proposed to effectively estimate $\x^*(\w)$ before updating $\w$, but anyway the less precise estimate slowdowns the convergence. The second obstacle is that hypergradient calculation requires second-order
derivatives of $\f_i(\w,\x)$. Many existing methods~\citet{chen2022single,dagreou2022framework} use an explicit second-order derivation of $\f_i(\w,\x)$ with a major focus on efficiently estimating its Jacobian and inverse Hessian, which is computationally expensive itself, but also dramatically increases the communication cost in a distributed setup.
A number of methods~\citet{chen2022single,li2022fully,dagreou2022framework} avoid directly estimating its second-order computation and only use the first-order information of both upper and lower objectives, but they still have high communication costs and do not exploit our assumptions.
For a more detailed review of BLO, we refer to~\citet{zhang2023introduction,liu2021investigating,chen2022gradient}.

\section{\Algn: Merit-Based Federated Learning For Diverse Datasets}\label{sec:meritfed}
Recall that the primary objective the target client seeks to solve is given by \eqref{eq:main_problem} where $n$ workers are connected with a parameter-server. Standard Parallel \algname{SGD}
\begin{equation}
  \textstyle  x^{t+1} = x^t - \frac{\gamma}{n}\sum_{i=1}^n g_i(x^t, \xiv_i), \label{eq:SGD}
\end{equation}
where $g_i(x^t, \xiv_i)$ denotes a stochastic gradient (unbiased estimate of $\nabla f_i(x^t)$) received from client $i$, cannot solve problem \eqref{eq:main_problem} in general, since workers $\{2,\ldots, n\}$ do not necessarily have the same data distribution as the target client. This issue can be solved if we modify the method as follows:
\begin{equation}
  \textstyle  x^{t+1} = x^t - \frac{\gamma}{|\cG|}\sum_{i\in \cG} g_i(x^t, \xiv_i), \label{eq:SGD_ideal}
\end{equation}
where $\cG$ denotes the set of workers that have the same data distribution as the target worker. However, the group $\cG$ is not known in advance. This aspect makes the method from \eqref{eq:SGD_ideal} impractical. Moreover, this method ignores potentially useful vectors received from the workers having different yet similar data distributions.

\begin{algorithm}[t]
    \caption{\Algn: Merit-based Federated Learning for Diverse Datasets}\label{alg:meritfed}
    \begin{algorithmic}[1] 
        \State {\bfseries Input:} Starting point $\x^0 \in \R^d$, stepsize $\gamma > 0$
        \For{$t=0,...$}
            \State server sends $\x^t$ to each worker
            \ForAll{\textbf{workers} $i= 1,\dots,\n$ \textbf{in parallel}}
                \State compute stochastic gradient 
                $\g_i(\x^t, \xiv_i)$  from local data and \textbf{send} $\g_i(\x^t, \xiv_i)$ to the server 
            \EndFor            
        \State
        \label{lst:line:aux_problem} 
        $\w^{t+1} \approx \argmin\limits_{w\in \Delta_1^n}f\left(x^t - \gamma \sum\limits_{i=1}^n \w_i  \g_i(\x^t, \xiv_i)\right)$
        \State $\x^{t+1} = \x^t - \gamma \sum\limits_{i=1}^n \w^{t+1}_i  \g_i(\x^t, \xiv_i)$. 
        \EndFor
    \end{algorithmic}
\end{algorithm}
\subsection{The Proposed Method}

We develop Merit-based Federated Learning for Diverse Datasets (\Algn; see Algorithm~\ref{alg:meritfed}) aimed at solving~\eqref{eq:main_problem} and safely gathering all potential benefits from collaboration with other clients. As in Parallel \algname{SGD} all clients are required to send the stochastic gradients to the server. However, in contrast to uniform averaging of the received stochastic gradients, \Algn uses the weights $w^t$ from the unit simplex $\Delta_1^n$ that are updated at each iteration. In particular, the new vector of weights $w^{t+1} \in \R^n$ at iteration $t$ approximates $\argmin_{w\in \Delta_1^n}f(x^t - \gamma \sum_{i=1}^n \w_i  \g_i(\x^t, \xiv_i))$. Then, the server uses the obtained weights for averaging the stochastic gradients and updating $x^t$. 

\subsection{Auxiliary Problem in Line~\ref{lst:line:aux_problem}}\label{s:aux}
In general, solving the problem in Line~\ref{lst:line:aux_problem} is not easier than solving the original problem \eqref{eq:main_problem}. Therefore, we present two particular approaches for efficiently addressing this problem.

{\bf Approach 1: use fresh data.} Let us assume that the target client can obtain new samples from distribution $\cD_1$ at any moment in time. To avoid any risk of compromising clients' privacy, the target client dataset should be stored only on the target client, and stochastic gradients received from other clients cannot be directly sent to the target client. To satisfy these requirements, one can approximate 
\begin{equation}\textstyle
    \argmin_{\w \in \Delta_1^n} \left\{\varphi_t(w) \equiv f \left(x^t - \gamma \sum_{i=1}^n \w_i  \g_i(\x^t, \xiv_i)\right)\right\}
\end{equation}
using \emph{zeroth-order}\footnote{In this case, the server can ask the target client to evaluate loss values at the required points without sending the stochastic gradients received from other workers.} Mirror Descent (or its accelerated version) \citet{duchi2015optimal, shamir2017optimal, gasnikov2022power}: 
\begin{equation}\textstyle
    w^{k+1} = \argmin_{w \in \Delta_1^n}\left\{\alpha\langle \tilde{g}^k, w\rangle + D_r(w, w^k)\right\}, \label{eq:zo_MD}
\end{equation}
where $\alpha > 0$ is the stepsize, $\tilde{g}^k$ is a finite-difference approximation of the directional derivative of sampled function
\begin{equation}\textstyle
    \varphi_{t,\xi^k}(w) \eqdef f_{\xi^k} \left(x^t - \gamma \sum_{i=1}^n \w_i  \g_i(\x^t, \xiv_i)\right),
\end{equation} 
where $\xi^k$ is a fresh sample from the distribution $\cD_1$ independent from all previous steps of the method, e.g., one can use $\tilde{g}^k = \frac{n(\varphi_{t,\xi^k}(w^k + he) - \varphi_{t,\xi^k}(w^k - he))}{2h}$ 
for $h > 0$ and $e$ being sampled from the uniform distribution on the unit Euclidean sphere, and $D_r(w,w^k) = r(w) - r(w^k) - \langle \nabla r(w^k), w - w^k \rangle$ is the Bregman divergence associated with a $1$-strongly convex function $r$. Although, typically, the oracle complexity bounds for gradient-free methods have $\cO(n)$ dependence on the problem dimension \citet{gasnikov2022randomized}, one can get just $\cO(\log^2(n))$, in the case of the optimization over the probability simplex \citet{shamir2017optimal, gasnikov2022power}. More precisely, if $f$ is $M_2$-Lipschitz w.r.t.\ $\ell_2$-norm and convex, then one can achieve $\E[\varphi_t(w) - \varphi_t(w^*)] \leq \delta$ using $\cO(\nicefrac{M_2^2\log^2(n)}{\delta^2})$ computations of $\varphi$, where $R$ is $\ell_1$-distance between the starting point and the solution \citet{gasnikov2022power} and prox-function is chosen as $r(w) = \sum_{i=1}^n w_i\log(w_i)$, which is $1$-strongly convex w.r.t.\ $\ell_1$-norm.
\\[5pt]
{\bf Approach 2: use additional validation data.} Alternatively, one can assume that the target client has an additional validation dataset $\widehat{D}$ sampled from $\cD_1$. Then, instead of function $f$ in Line~\ref{lst:line:aux_problem}, one can approximately minimize
\begin{equation}\textstyle
    \widehat f(x) = \frac{1}{|\widehat D|}\sum_{\xi \in \widehat D} f_{\xi}(x), \label{eq:validation_loss}
\end{equation}
which under certain conditions provably approximates the original function $f(x)$ with any predefined accuracy if the dataset $\widehat D$ is sufficiently large \citet{shalev2009stochastic, feldman2019high}. More precisely, the worst-case guarantees (e.g.,~\citet{liu2024new}) imply that to guarantee $\mathbb{E}[f(\widehat x^\ast) - f(x^\ast)] \leq \delta$, where $\widehat x^\ast \in \arg\min_{x\in \mathbb{R}^d} \widehat f(x)$ and $x^\ast \in \arg\min_{x\in \mathbb{R}^d} f(x)$, the validation dataset should be of the size $|\mathcal{\widehat D}| \sim \max\left\lbrace \nicefrac{L}{\mu}, \nicefrac{1}{\mu\delta} \right\rbrace$ under the assumption that $f_\xi (x)$ is $\mu$-strongly convex. However, as we observe in our experiments, \algname{MeritFed} works well even with a relatively small size of the validation dataset for non-convex problems.

{\bf Memory usage.} It is also worth mentioning that \algname{MeritFed} requires the server to store $n$ vectors at each iteration for solving the problem in Line~\ref{lst:line:aux_problem}. While standard \algname{SGD} does not require such a memory, closely related methods --- \algname{FedAdp} and \algname{TAWT} --- also require the server to store $n$ vectors for the computation of the weights for aggregation. However, for modern servers, this is not an issue.

\section{Convergence Analysis}\label{sec:convergence}

In our analysis, we rely on the standard assumptions for non-convex optimization literature.

\begin{assumption}\label{as:bounded-var}
 For all $i\in \cG$ the stochastic gradient $\g_i(\x, \xiv_i)$ is an unbiased estimator of $\nabla\f_i(\x)$ with bounded variance, i.e., $\E_{\xiv_i} [\g_i(\x, \xiv_i)] = \nabla\f_i(\x)$ and for some $\sigma \geq 0$
    \begin{equation}  \label{eq:xi-var}
    \textstyle
        \E_{\xiv_i} \left[\norm*{ \g_i(\x, \xiv_i) -  \nabla\f_i(\x)}^2\right] \leq \sigma^2.
    \end{equation}
\end{assumption}

The above assumption is known as the bounded variance assumption. It is classical for the analysis of stochastic optimization methods, e.g., see \citet{nemirovski2009robust, juditsky2011solving}.

Next, we assume the smoothness of the objective.
\begin{assumption}\label{as:lipschitzness} 
    $\f$ is $L$-smooth, i.e., $\forall\; \x,\y \in \R^d$
	\begin{equation}\label{eq:lipschitzness}\textstyle
		\f\rbr{\x} \le \f\rbr{\y} + \inp*{\nabla \f\rbr{\y}}{\x - \y} + \frac{L}{2}\norm*{\x - \y}^2. \tag{Lip}
	\end{equation}
\end{assumption}
We also make the following (optional) assumption called Polyak-{\L}ojasiewicz (P{\L}) condition \citet{polyak1963gradient, lojasiewicz1963topological}.
\begin{assumption}\label{as:pl}
    $\f$ satisfies Polyak-{\L}ojasiewicz (P{\L}) condition with parameter $\mu$, i.e., for $\mu \ge 0$
    \begin{equation} \label{eq:pl}\textstyle
        \f^* \ge \f\rbr*{\x} - \frac{1}{2\mu}\norm*{\nabla \f(\x)}^{2},\quad \forall\;\x \in \R^d. \tag{PL}
    \end{equation}
\end{assumption}
This assumption belongs to the class of structured non-convexity assumptions allowing linear convergence for first-order methods such as Gradient Descent \citet{necoara2019linear}.


The main result for \Algn is given below (see the proof in Appendix~\ref{appendix:proofs}).
\begin{theorem}\label{thm:main_result}
    Let Assumptions~\ref{as:bounded-var} and~\ref{as:lipschitzness} hold. Then after $T$ iterations, \Algn with $\gamma \leq \frac{1}{2L}$ outputs $\x^{i}$, $i=0,\cdots, T-1$ such that
    \begin{equation*}\textstyle
        \frac{1}{T}\sum_{t=0}^{T-1}\E \norm*{\nabla\f\rbr*{\x^{t}}}^{2}
        \le \frac{2\rbr*{\f\rbr*{\x^{0}} - \f\rbr*{\x^{*}}}}{T\gamma} + \frac{2\sigma^2 \gamma L}{\gn} + \frac{2 \delta}{\gamma},
    \end{equation*}
    where $\delta$ is the accuracy of solving the problem in Line~\ref{lst:line:aux_problem} and $G = |\cG|$. 
    Moreover if Assumption~\ref{as:pl} additionally holds, then after $T$ iterations of \Algn with $\gamma \leq \frac{1}{2L}$ outputs $\x^{T}$ such that
    \begin{equation*}\textstyle
        \E\f\rbr*{\x^{T}} - \f^* \le \rbr*{1 - \gamma\mu}^T \rbr*{\f\rbr*{\x^{0}} - \f^*}  + \frac{\sigma^2 \gamma L}{\mu\gn} + \frac{\delta}{\gamma\mu}.
    \end{equation*}
\end{theorem}


If $\delta$ is sufficiently small, then the above result matches the known results for Parallel \algname{SGD} \citet{ghadimi2013stochastic, karimi2016linear, khaled2022better} that uniformly averages only the workers from the group $\cG$, i.e., those workers that have data distribution $\cD_1$ (see the method in \eqref{eq:SGD_ideal}). More precisely, we see a linear speed-up of $\nicefrac{1}{G}$ in the obtained convergence rates. However, \Algn does not require the knowledge of the workers that have the same distribution. Moreover, as we see in our numerical experiments, \Algn can converge even better when there exist workers having different but close data distribution, and it is not necessary to solve the auxiliary problem in Line~\ref{lst:line:aux_problem} with high precision.


\section{Experiments} \label{sec:numerical_results}


Since the literature on FL is very rich, we focus only on the closely related methods, i.e., the methods that satisfy two criteria: (i) they solve the same problem as we consider in our work~\ref{eq:main_problem}, and (ii) have theoretical convergence guarantees. That is, we evaluate the performance of proposed methods in comparison with \algname{FedAdp}~\citet{wu2021fast}, \algname{TAWT}~\citet{chen2021weighted}, 
and \algname{FedProx}~\citet{li2020federatedopt} (\algname{FedProx} reduces to \algname{FedAvg} if there are no local steps, that is the setup for \algname{MeritFed}). We also compare standard SGD with uniform weights (labeled as \algname{SGD Full}\footnote{Although, \algname{FedProx} and \algname{SGD Full} are designed for standard empirical risk minimization, we consider these methods as standard baselines.}), \algname{SGD} that accumulates only gradients from clients with the target distribution (\algname{SGD Ideal}) and two versions of our algorithm. The first one, labeled as \algname{MeritFed SMD}, samples gradient for the Mirror Descent subroutine in contrast to the other one, labeled as \algname{MeritFed MD}, that uses the full dataset (additional or train) to calculate gradient for Mirror Descent step. We use only $10$ Mirror Descent steps for solving the auxiliary problem from Line~\ref{lst:line:aux_problem} since it was sufficient to achieve good enough results in our experiments. In addition, we present the evolution of weighs (if applicable) using heat-map plots. In the main text, we show the results for the case when the additional validation dataset is available for the problem in Line~\ref{lst:line:aux_problem}. Additional experiments with the usage of train data for the problem in Line~\ref{lst:line:aux_problem}, with the presence of Byzantine participants and with more workers, are provided in the appendix. Our code is available at \url{https://anonymous.4open.science/r/86315}. We use internal cluster with the following hardware: AMD EPYC 7552 48-Core Processor, 512GiB RAM, NVIDIA A100 40GB GPU, 200Gb user storage space.

\begin{figure}[t]
    \begin{minipage}[htp]{0.325\textwidth}
        \centering
        \includegraphics[width=1\linewidth]{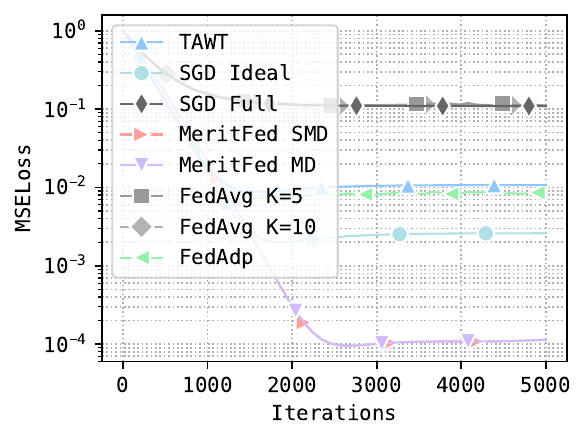}
        \vspace{-0.8cm}
        \caption{Mean Estimation: $\mu = 0.001$, MD learning rate = 3.5.}
        \vspace{-0.3cm}
        \label{fig:emean1}
    \end{minipage}
\hfill
    \begin{minipage}[htp]{0.325\textwidth}
        \centering
        \includegraphics[width=1\linewidth]{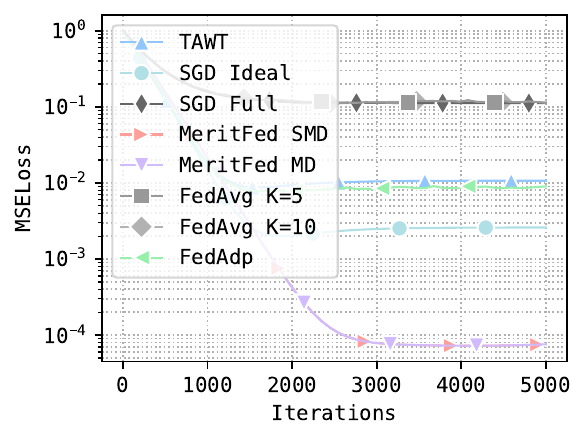}
        \vspace{-0.8cm}
        \caption{Mean Estimation: $\mu = 0.01$, MD learning rate = 4.5.}
        \vspace{-0.3cm}
        \label{fig:emean2}
    \end{minipage}
\hfill
    \begin{minipage}[htp]{0.325\textwidth}
        \centering
        \includegraphics[width=1\linewidth]{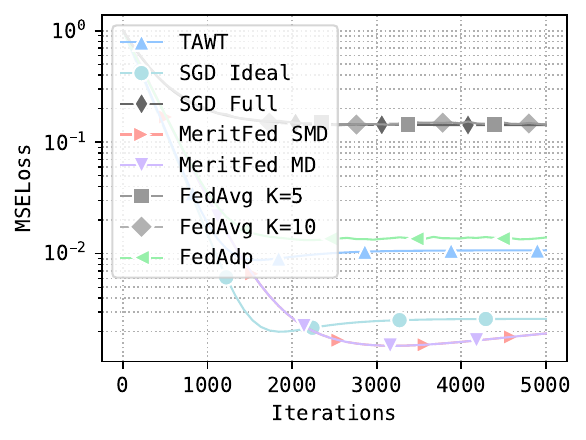}
        \vspace{-0.8cm}
        \caption{Mean Estimation: $\mu = 0.1$, MD learning rate = 12.5.}
        \vspace{-0.3cm}
        \label{fig:emean3}
    \end{minipage}
\end{figure}




{\bf Mean estimation.} We start with the mean estimation problem, i.e., finding such a vector that minimizes the mean squared distance to the data samples. More formally, the goal is to solve
\begin{equation}\textstyle
    \min_{x\in \R^d} \E_{\xi \sim \cD_1}\|x - \xi\|^2, \notag
\end{equation}
that has the optimum at $x^* = \E_{\xi \sim \cD_1}[\xi]$. We consider $\cD_1 = \cN(0, \mI)$ and also two other distributions from where some clients also get samples: $\cD_2 = \cN(\mu\mathbf{1}, \mI)$ and $\cD_3 = \cN(e, \mI)$, where $\mathbf{1} = (1,1,\ldots,1)^\top \in \R^d$, $\mu > 0$ is a parameter, and $e$ is some vector that we obtain in advance via sampling uniformly at random from the unit Euclidean sphere. We consider $150$ clients with data distributed as follows: the first $5$ workers have data from $\cD_1$ (the first group of clients), the next $95$ workers have data from $\cD_2$ (the second group of clients), and the remaining $50$ clients have data from $\cD_3$ (the third group of clients). Each client has $1000$ samples from the corresponding distribution, and the target client has additional $1000$ samples for validation, i.e., for solving the problem in Line~\ref{lst:line:aux_problem}. The dimension of the problem is $d = 10$.  Parameters that are the same for all experiments: number of peers $= 150$, number of samples $= 1000$, batch size $= 100$, learning rate $= 0.01$, number of steps for Mirror Descent $= 50$. For \algname{FedAvg}, the number of sampled clients $K$ is chosen from the set $\{5, 10\}$.




We consider three cases: $\mu = 0.001$, $0.01$, $0.1$. The smaller $\mu$ is, the closer $\cD_2$ is to $\cD_1$ and, thus, the more beneficial the samples from the second group are. Therefore, for small $\mu$, we expect to see that \Algn outperforms \algname{SGD Ideal}. Moreover, since the workers from the third group have quite different data distribution, \algname{SGD Full} is expected to work worse than other baselines.

The results are presented in Figures~\ref{fig:emean1}-\ref{fig:emean3}. They fit the described intuition and our theory well:
the workers from the second group are beneficial 
(since their distributions are close enough to the distribution of the target client). Indeed, \Algn achieves better optimization error (due to the smaller variance because of the averaging with more workers). 
However, when 
the dissimilarity between distributions is large
the second group becomes less useful for the training, and \Algn has comparable performance to \algname{SGD Ideal} and consistently outperforms other methods.



\begin{figure*}[t]
    \begin{minipage}[htp]{0.24\textwidth}
        \centering
        \includegraphics[width=1\linewidth]{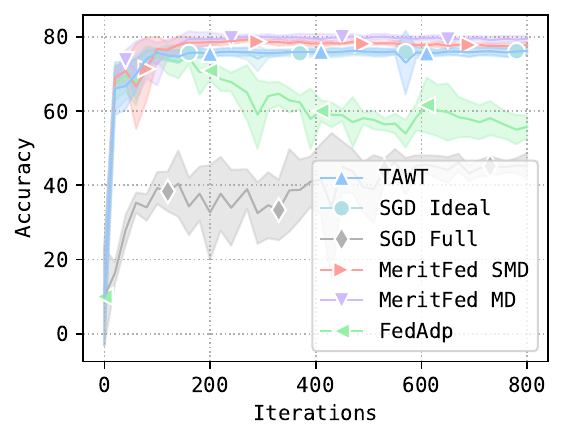}
        \includegraphics[width=1\linewidth]{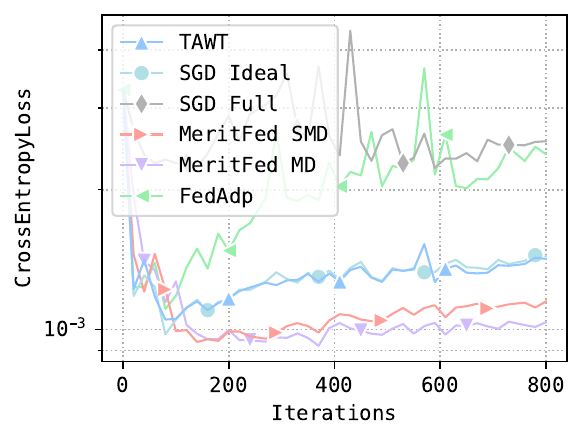}
        \includegraphics[width=1\linewidth]{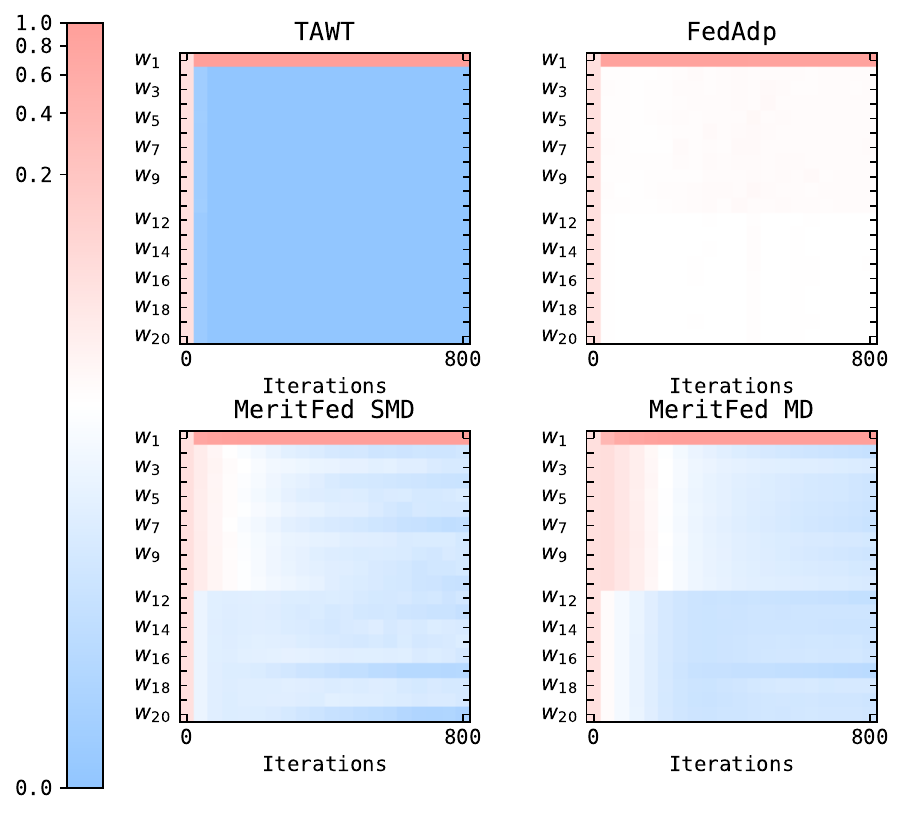}
        \vspace{-0.6cm}
        \caption{CIFAR10 (extra val.): $\alpha = 0.5$}
        \vspace{-0.2cm}
        \label{fig:cifar-val-0.5}
    \end{minipage}
\hfill
    \begin{minipage}[htp]{0.24\textwidth}
        \centering
        \includegraphics[width=1\linewidth]{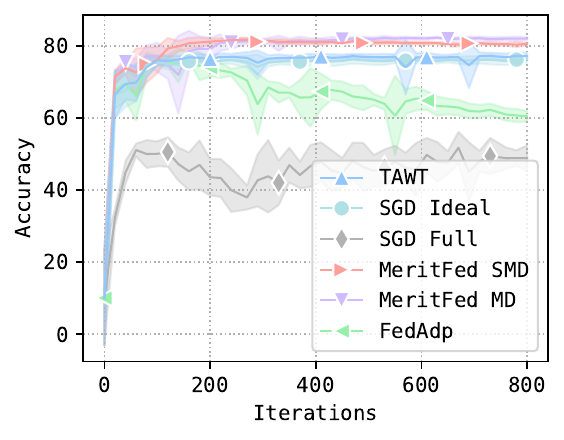}
        \includegraphics[width=1\linewidth]{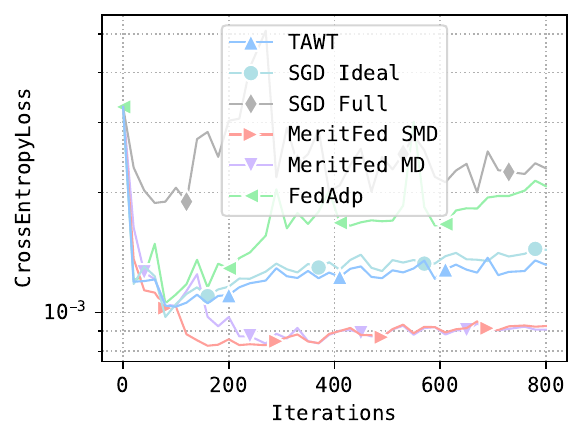}
        \includegraphics[width=1\linewidth]{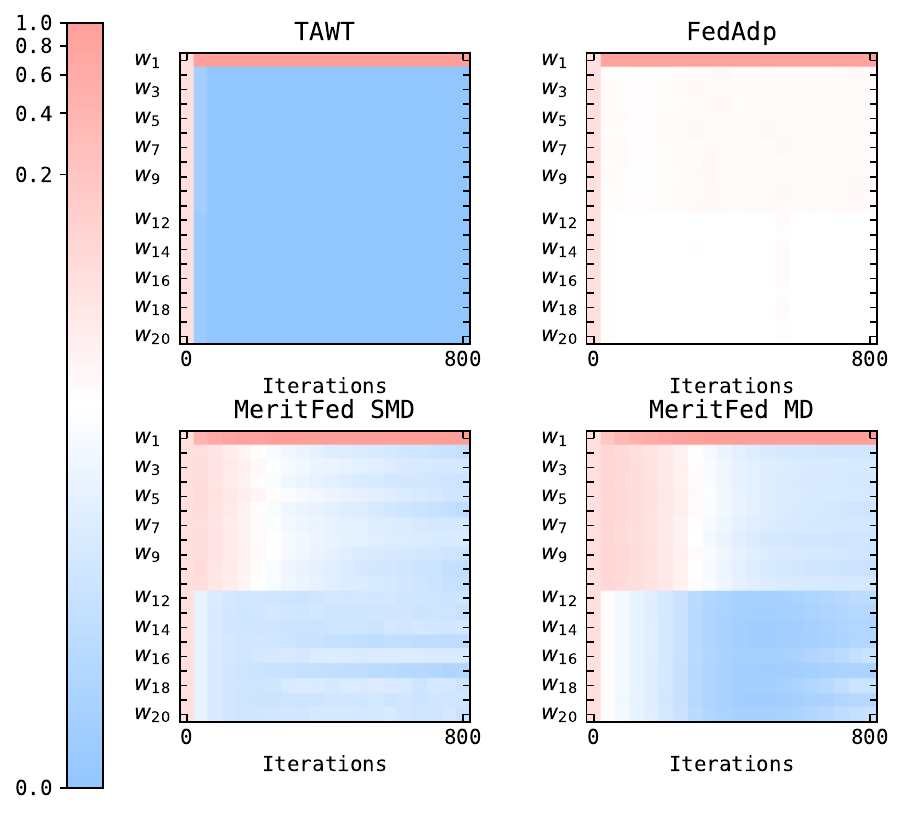}
        \vspace{-0.6cm}
        \caption{CIFAR10 (extra val.): $\alpha = 0.7$}
        \vspace{-0.2cm}
        \label{fig:cifar-val-0.7}
    \end{minipage}
\hfill
    \begin{minipage}[htp]{0.24\textwidth}
        \centering
        \includegraphics[width=1\linewidth]{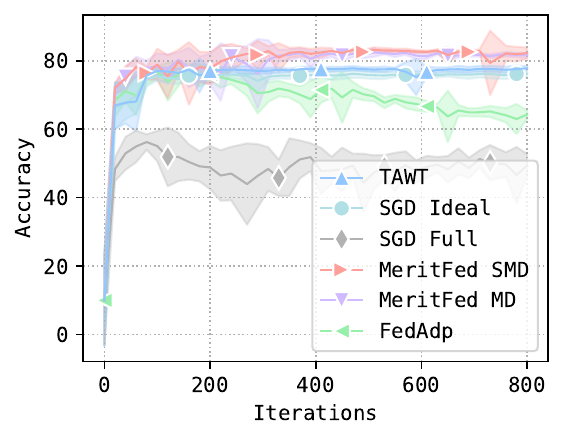}
        \includegraphics[width=1\linewidth]{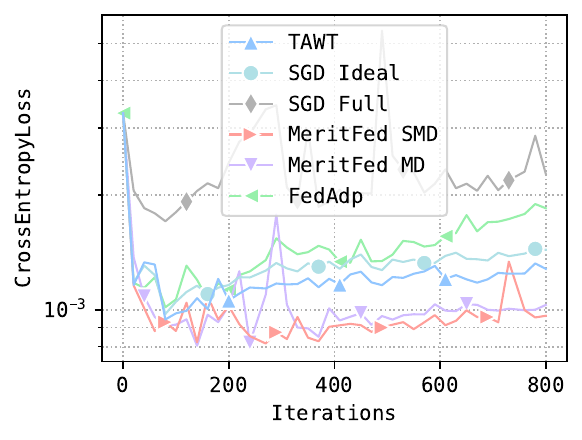}
        \includegraphics[width=1\linewidth]{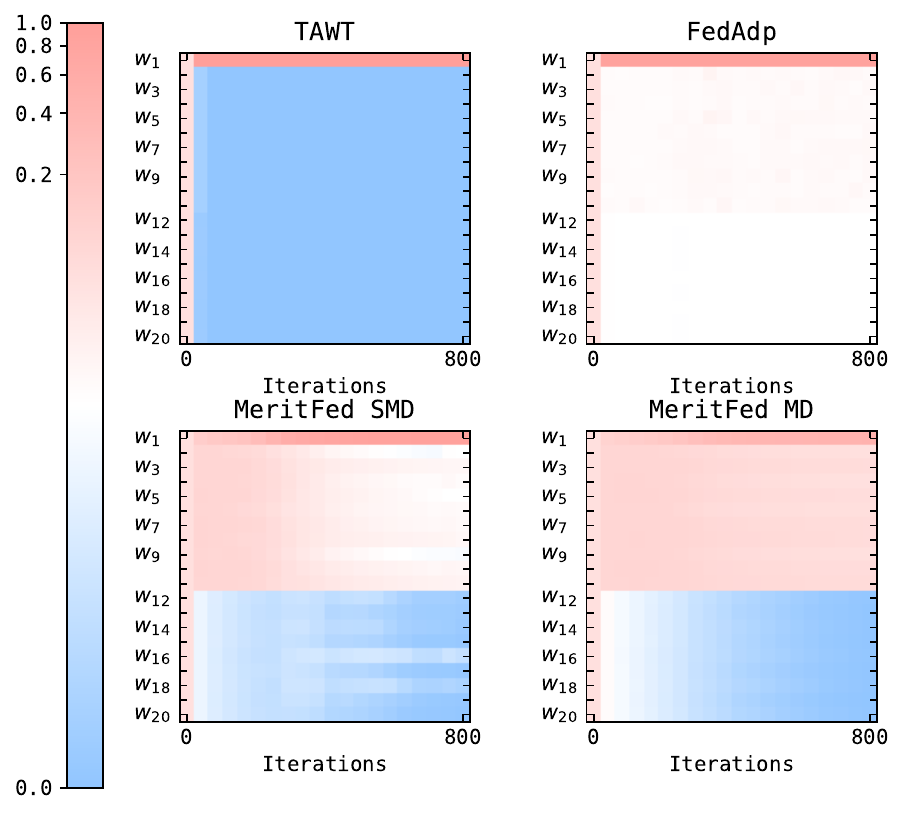}
        \vspace{-0.6cm}
        \caption{CIFAR10 (extra val.): $\alpha = 0.9$}
        \vspace{-0.2cm}
        \label{fig:cifar-val-0.9}
    \end{minipage}
\hfill
    \begin{minipage}[htp]{0.24\textwidth}
        \centering
        \includegraphics[width=1\linewidth]{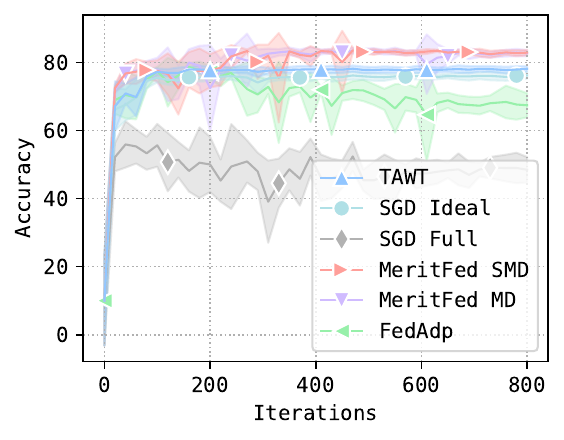}
        \includegraphics[width=1\linewidth]{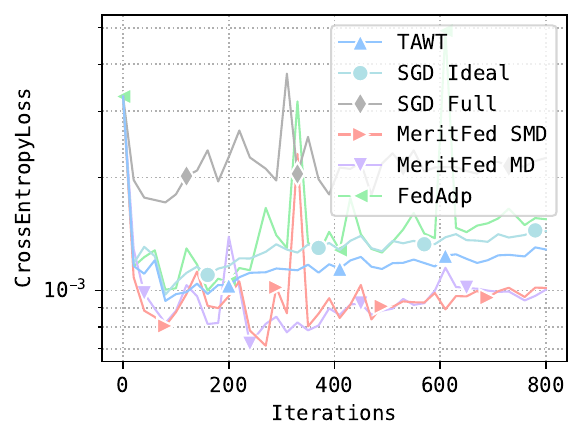}
        \includegraphics[width=1\linewidth]{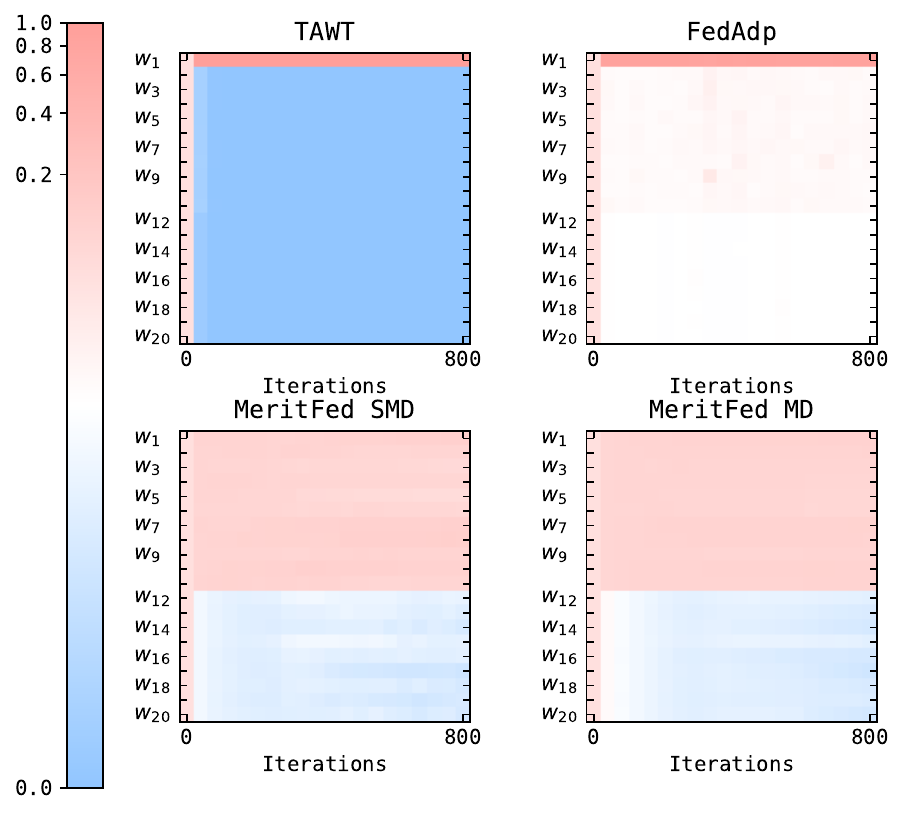}
        \vspace{-0.6cm}
        \caption{CIFAR10 (extra val.): $\alpha = 0.99$.}
        \vspace{-0.2cm}
        \label{fig:cifar-val-0.99}
    \end{minipage}
\end{figure*}

\begin{figure*}[t]
    \begin{minipage}[htp]{0.24\textwidth}
        \centering
        \includegraphics[width=1\linewidth]{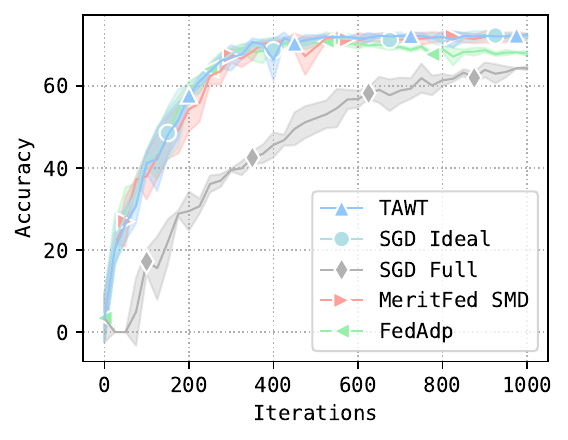}
        \includegraphics[width=1\linewidth]{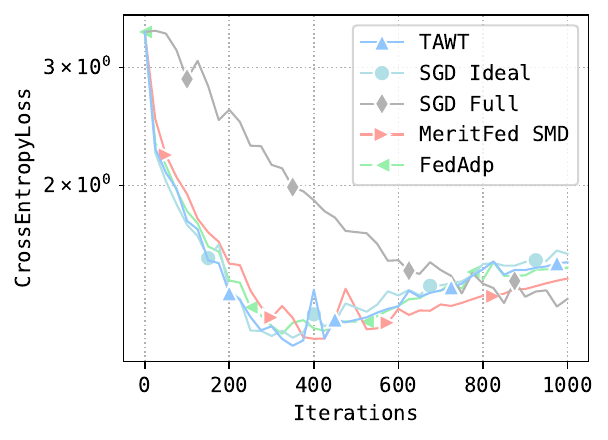}
        \includegraphics[width=1\linewidth]{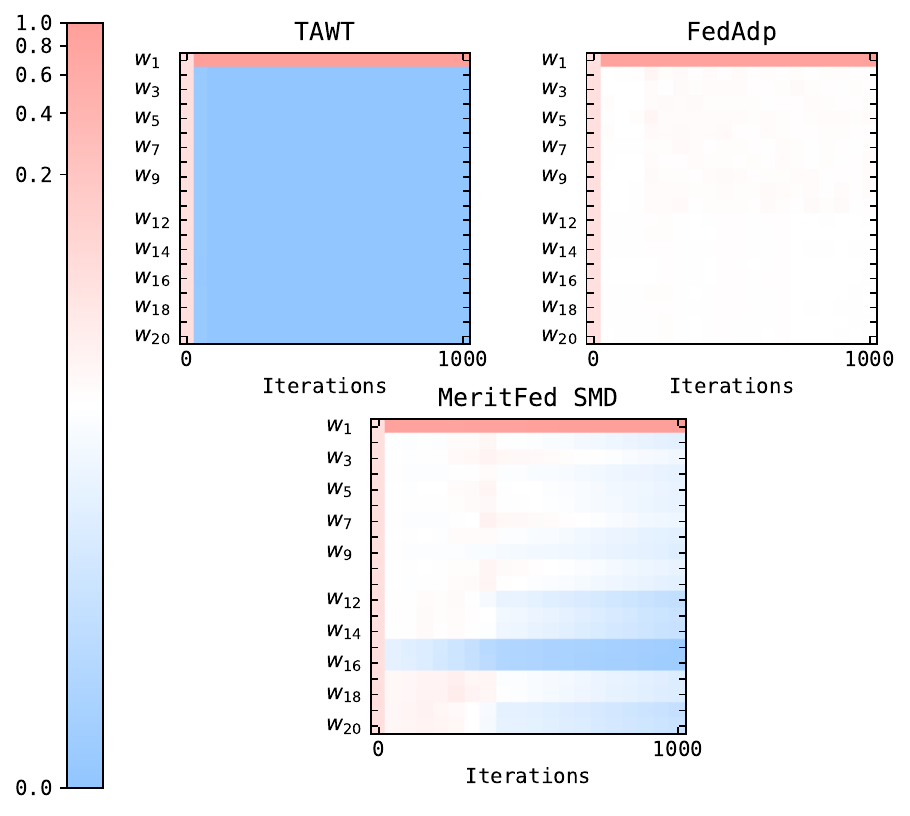}
        \vspace{-0.6cm}
        \caption{GoEmotions (extra val.): $\alpha = 0.5$}
        \vspace{-0.3cm}
        \label{wfig:bert-val-0.5}
    \end{minipage}
\hfill
    \begin{minipage}[htp]{0.24\textwidth}
        \centering
        \includegraphics[width=1\linewidth]{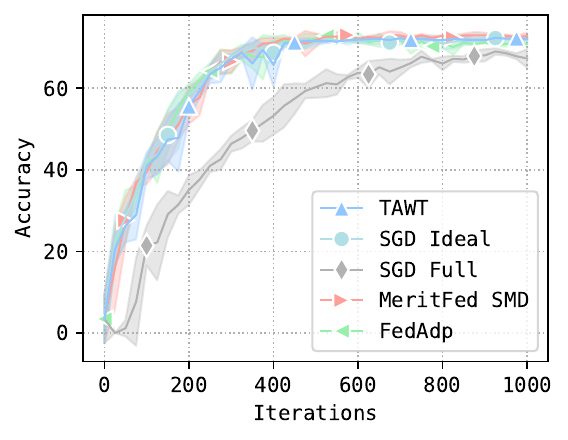}
        \includegraphics[width=1\linewidth]{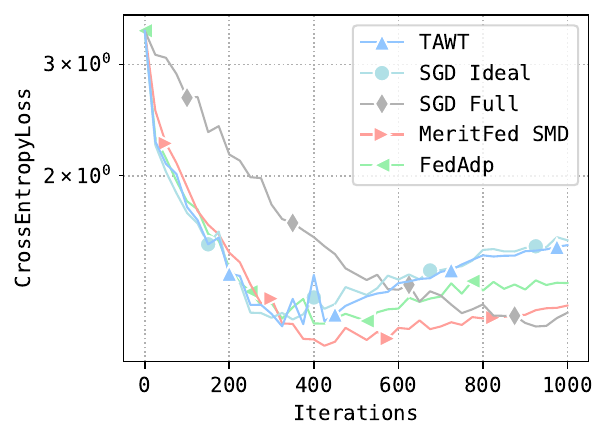}
        \includegraphics[width=1\linewidth]{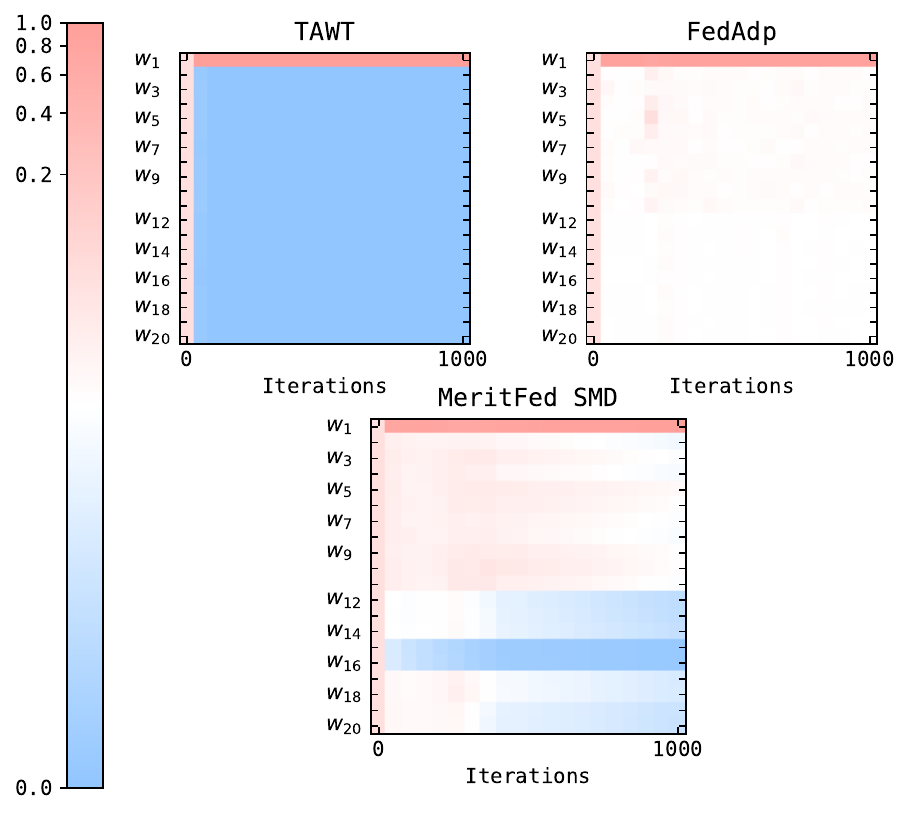}
        \vspace{-0.6cm}
        \caption{GoEmotions (extra val.): $\alpha = 0.7$}
        \vspace{-0.3cm}
        \label{wfig:bert-val-0.7}
    \end{minipage}
\hfill
    \begin{minipage}[htp]{0.24\textwidth}
        \centering
        \includegraphics[width=1\linewidth]{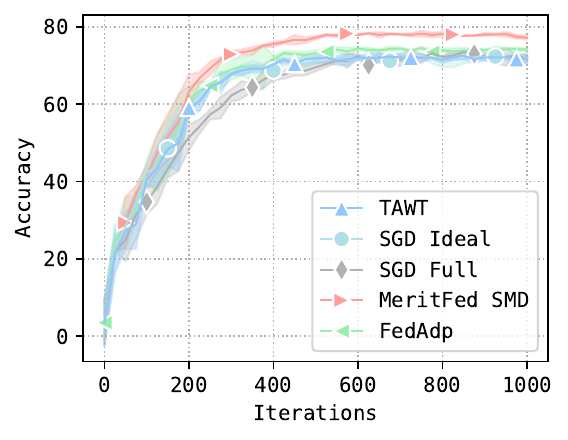}
        \includegraphics[width=1\linewidth]{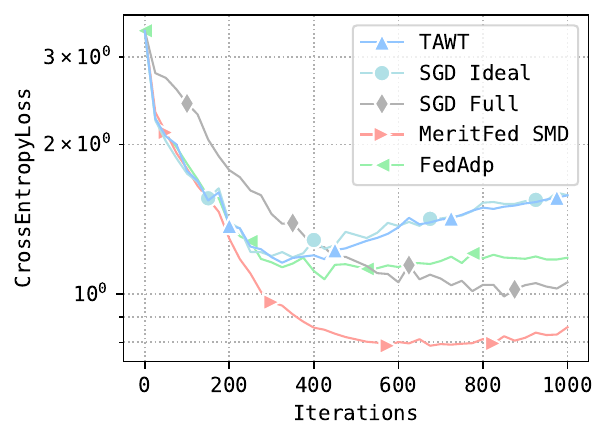}
        \includegraphics[width=1\linewidth]{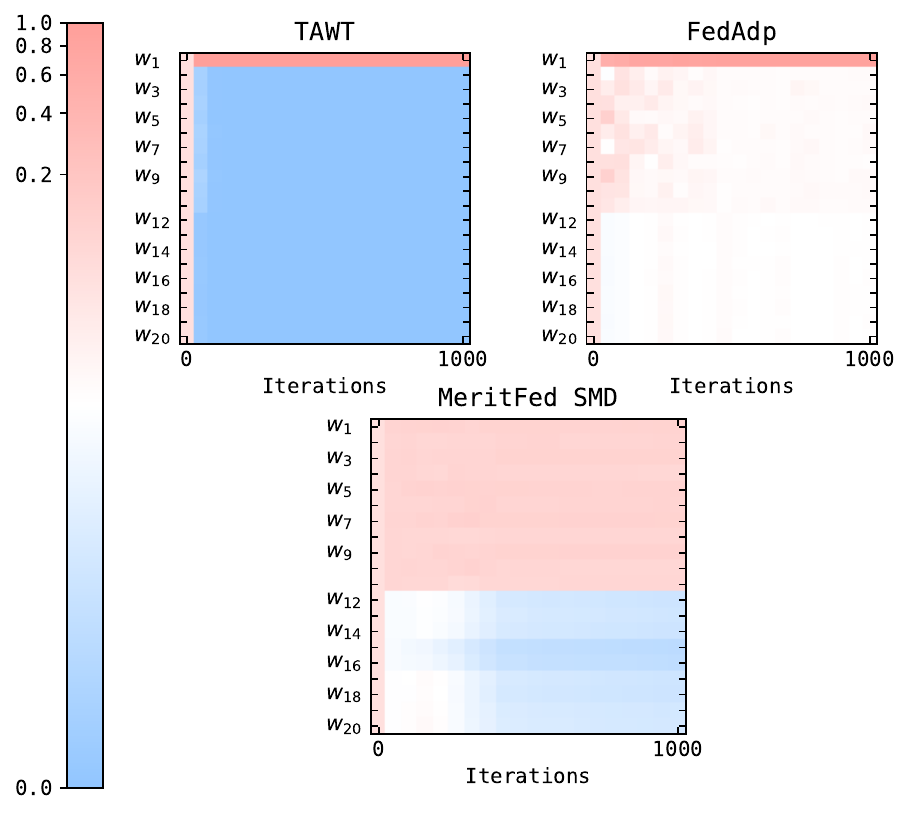}
        \vspace{-0.6cm}
        \caption{GoEmotions (extra val.): $\alpha = 0.9$}
        \vspace{-0.3cm}
        \label{wfig:bert-val-0.9}
    \end{minipage}
\hfill
    \begin{minipage}[htp]{0.24\textwidth}
        \centering
        \includegraphics[width=1\linewidth]{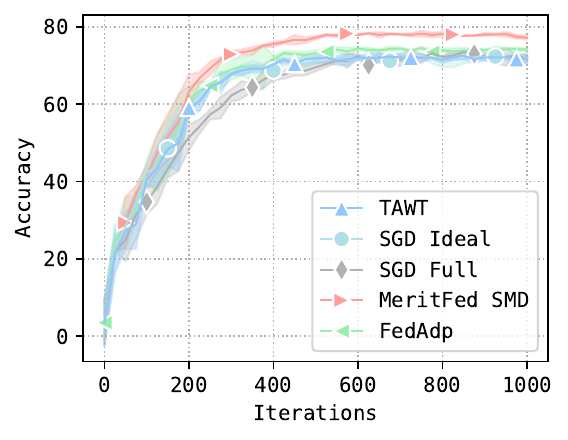}
        \includegraphics[width=1\linewidth]{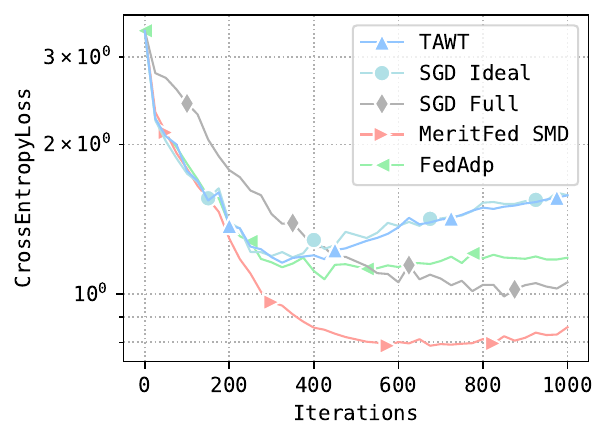}
        \includegraphics[width=1\linewidth]{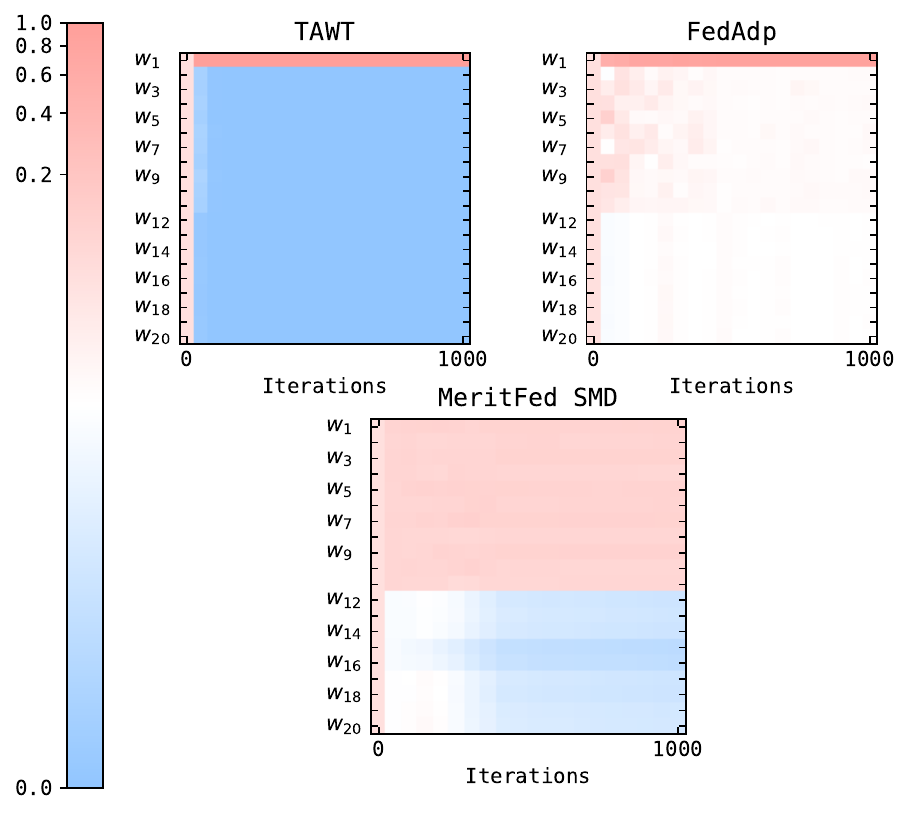}
        \vspace{-0.6cm}
        \caption{GoEmotions (extra val.): $\alpha = 0.99$}
        \vspace{-0.2cm}
        \label{wfig:bert-val-0.99}
    \end{minipage}
\end{figure*}

{\bf Image classification: CIFAR10 + ResNet18.} This part is devoted to image classification on the CIFAR10~\citet{krizhevsky2009learning} dataset using ResNet18~\citet{he2016deep} model and cross-entropy loss. We consider $20$ clients with data distributed as follows: the first worker has data from $\cD_1$ (the first group of clients), the next $10$ workers have data from $\cD_2$ (the second group of clients), and the remaining $9$ clients have data from $\cD_3$ (the third group of clients). Specifically, the target client's objective is to classify the first three classes: 0, 1, and 2. This client possesses data with these three labels. The following ten workers (second group) also have datasets where a proportion, denoted by $\alpha \in (0,1]$, consists of classes from the set ${0,1,2}$, while the remaining $1-\alpha$ portion includes classes from the set ${3,4,5}$.
The remaining clients (third group) have data from the rest, e.g., ${6,7,8,9}$ labeled. The data is randomly distributed among clients without overlaps, adhering to the aforementioned label restrictions.
For \algname{MeritFed} each worker calculates stochastic gradient using a batch size of $75$; then the server uses Mirror Descent (or its stochastic version) with a batch-size of $90$ (in case of stochastic version) and a learning rate of $0.1$ to update weights of aggregation, and then performs a model parameters update with a learning rate of $0.01$.
We normalize images (similarly to~\citet{horvath2020better}).
Since an additional validation dataset can be used by  \Algn, we cut 300 samples of each target class (${0, 1, 2}$) off from the test data. Accuracy and loss are calculated on the rest of the test data, including labels $0$, $1$, and $2$, modeling the case when the target client aims to classify samples with these labels.

The results are provided in Figures~\ref{fig:cifar-val-0.5}-\ref{fig:cifar-val-0.99}, where we show how accuracy and cross-entropy loss change for different methods and different values of $\alpha$, which measures the similarity between data distributions of the target client and the second group of clients, and the evolution of the aggregation weights. In all settings, \algname{MeritFed} outperforms \algname{SGD Ideal} and other baselines regardless of $\alpha$. In all cases, the weights are almost the same for all workers during the few initial steps (even if workers have quite different distributions like for the last nine clients). This phenomenon can be explained as follows: if we have two different convex functions with different optima (e.g., two quadratic functions), then for a far enough starting point, the gradients of those functions will point roughly in the same direction. Therefore, during a few initial steps, both gradients are useful and the method gives noticeable weights to both. However, once the method comes closer to the optima, the gradients become noticeably different, and after a certain stage, the gradient of the second function no longer points closely towards the optimum of the first function. Therefore, starting from this stage, \algname{MeritFed} assigns a smaller weight to the gradient of the second function. Going back to~Figures~\ref{fig:cifar-val-0.5}-\ref{fig:cifar-val-0.99}, we see a similar behavior: for $\alpha = 0.5$, the advantages of collaboration with clients $2$-$11$ disappear after a certain stage since the method reaches the region where two distributions become noticeably different. In contrast, when $\alpha = 0.99$, those workers have a very close distribution to the target worker, and therefore, their stochastic gradients remain useful during the whole learning process. \algname{FedAdp} is biased to the target client and assigns almost identical weights to either clients with similar or dissimilar distributions, which results in an accuracy decrease at the end of the training, in contrast to \Algn, which tracks and maintains less weights to non-beneficial clients. \algname{TAWT} is much more biased to the target client, which makes it almost identical to \algname{SGD Ideal}.



{\bf Texts classification: GoEmotions + BERT.} The next problem we consider is devoted to fine-tuning pretrained BERT~\citet{devlin2018bert}  model for emotions classification on the GoEmotions dataset~\citet{demszky2020goemotions}. 
The dataset consist of texts labeled with one or more of 28 emotions.
First of all, we form "truncated dataset" by cutting the dataset so that its each entry has the only label.

Then we use Ekman mapping~\citet{ekman1992argument} to split the data between clients. According to the mapping, 28 emotions can be mapped to 7 basic emotions. That is, we simulate a situation when the target client classifies only basic emotions, e.g., the target client has only emotions belonging to "joy" class and namely includes only "joy", "amusement", "approval", "excitement", "gratitude",  "love", "optimism", "relief", "pride", "admiration", "desire", "caring". The distribution of these sub-emotions is kept to be the same as the distribution of the truncated train dataset. Clients, that data are suppose to have similar distribution (second group -- next 10 clients), also has texts from base class "joy" and are labeled as one of the sub-emotion belonging to "joy".  The distribution of sub-emotions is also the same as the distribution of the truncated train dataset. These texts constitute an $\alpha$ portion of the total client's data. The other $1-\alpha$ portion of the texts is taken from "neutral" class.

The rest of clients (third group -- next 9 clients) are supposed to have different distribution and their data consist of either texts belonging to one of the other basic emotion, either mixed with neutral (if there is not enough texts to have a desired number of samples) or texts from "neutral" class only. Again, the distribution of sub-emotions is the same as the distribution of the truncated train dataset.
For \algname{MeritFed} each worker calculates stochastic gradient using a batch size of $40$; then the server uses Mirror Descent (or its stochastic version) with a batch-size of $30$ (in case of stochastic version) and a learning rate of $0.1$ to update weights of aggregation, and then performs a model parameters update with a learning rate of $0.01$.
The plots are averaged over 3 runs with different seeds. Additionally, accuracy plots show standard deviation.
The results are presented in Figures~\ref{wfig:bert-val-0.5}-\ref{wfig:bert-val-0.99}. The target client benefits from collaborating with clients from the second group and achieves better accuracy using \Algn. In general, the results are similar to the ones obtained for image classification.






\section{Conclusion}\label{sec:conclusion}

In this paper, we introduced a novel algorithm called Merit-based Federated Learning (\Algn) to address the challenges posed by the heterogeneous data distributions in federated learning (FL). We demonstrated that \Algn can effectively harness the advantages of distinct data distributions, control the detrimental effects of outlier clients, and promote collaborative learning. Our approach assigns adaptive aggregation weights to clients participating in FL training, allowing for faster convergence and potentially better generalization. The key contributions of this paper include the development of \Algn, provable convergence under mild assumptions, and the ability to utilize benefits from collaborating with clients having different but similar data distributions.

However, our work also has certain limitations. Firstly, (in theory) \Algn relies on the availability of validation data on the target client to solve the auxiliary problem. In practice, collecting and maintaining such validation data might not always be feasible or efficient. Secondly, the experiments in this paper were conducted with a limited number of clients and a moderate dataset size. Extending \Algn to large-scale FL scenarios with a substantial number of clients and massive datasets may pose scalability challenges. Addressing these limitations is part of our plan for future work. 

Furthermore,  \Algn serves as a foundation for numerous extensions and enhancements. Future research can explore topics such as acceleration techniques, adaptive or scaled optimization methods (e.g., variants akin to \algname{Adam}) on the server side, communication compression strategies, and the efficient implementation of similar collaborative learning approaches for all clients simultaneously. These directions will contribute to the continued development of federated learning methods, making them more efficient, robust, and applicable to a wide range of practical scenarios.

\bibliography{refs}
\bibliographystyle{apalike}

\newpage
\appendix
\onecolumn

\tableofcontents

\newpage

\section{Weights Update for \algname{TAWT} and \algname{FedAdp}}

{\bf TAWT.} A faithful implementation of \algname{TAWT} \citep{chen2021weighted} would require a costly evaluation of the inverse of the Hessian matrix $\sum_{t=1}^T \w_t \nabla^2 \f({\x^k})$ to calculate an approximation of hyper-gradient $g^{k}$.
Then $g^{k}$ is supposed to be used to run one step of Mirror Descent (with step size $\eta^k$) to update the weights:
\begin{equation}
    \label{eq:mirror_descent}
    \w_t^{k+1} = \frac{\w_t^k \exp\{-\eta^k g^{k}_t\} }{\sum_{t'=1}^T \w_{t'}^{k}\exp\{-\eta^k g^{k}_{t'}\}}.
\end{equation}

In practice, \citet{chen2021weighted} advise bypassing this step by replacing the Hessian-inverse-weighted dissimilarity measure with a cosine-similarity-based measure, i.e., to approximate $g^k_t$ by $-c \times \cS(\nabla \f_0({\x^k}), \nabla\f_t({\x^k}))$, where 
 \begin{align*}
    \label{eq8}
    \cS(a,b) & = \arccos{ \frac{\langle a,b \rangle}{\Vert a\Vert  \Vert b\Vert} } \tag{8}
\end{align*}
denotes the cosine similarity between two vectors.

{\bf FedAdp.} \algname{FedAdp} \citep{wu2021fast} uses a similar update rule for weights, but it additionally uses a non-linear mapping function (\textit{Gompertz function}) 
\begin{align*}
\label{eq13}
    \cG(\xi) = \alpha \left(1 - e^{-e^{-\alpha\xi}}\right)
\end{align*}
where $\xi$ is the \textit{smoothed angle} in \textit{radian}, $e$ denotes the exponential constant and $\alpha$ is a constant. By denoting $\cS^k_t = \cS(\nabla \f_0({\x^k}), \nabla\f_t({\x^k}))$ one can obtain \algname{FedAdp} weights update rule in the form 
\begin{align*}
    \w^k_t = \frac{e^{\cG( \cS^k_t )}}{\sum_{t^\prime=1}^{n} e^{\cG(\cS^k_t)}}.
\end{align*}

\newpage

\section{Proof of Theorem~\ref{thm:main_result}}\label{appendix:proofs}

\begin{theorem}\label{thm:main_result_appendix}
    Let Assumptions~\ref{as:bounded-var} and~\ref{as:lipschitzness} hold. Then after $T$ iterations of \Algn with $\gamma \leq \frac{1}{2L}$ outputs $\x^{i}$, $i=0,\cdots, T-1$ such that
    \begin{equation}
        \frac{1}{T}\sum_{t=0}^{T-1}\E \norm*{\nabla\f\rbr*{\x^{t}}}^{2}
        \le \frac{2\rbr*{\f\rbr*{\x^{0}} - \f\rbr*{\x^{*}}}}{T\gamma} + \frac{2\sigma^2 \gamma L}{\gn} + \frac{2 \delta}{\gamma}, \label{eq:non_cvx_result_appendix}
    \end{equation}
    where $\delta$ is the accuracy of solving the problem in Line~\ref{lst:line:aux_problem} and $G = |\cG|$. 
    Moreover if Assumption~\ref{as:pl} additionally holds, then after $T$ iterations of \Algn with $\gamma \leq \frac{1}{2L}$ outputs $\x^{T}$ such that
    \begin{equation}
        \E\f\rbr*{\x^{T}} - \f^* \le \rbr*{1 - \gamma\mu}^T \rbr*{\f\rbr*{\x^{0}} - \f^*}  + \frac{\sigma^2 \gamma L}{\mu\gn} + \frac{\delta T}{\gamma\mu}. \label{eq:PL_result_appendix}
    \end{equation}
\end{theorem}
\begin{proof}
    We write $\g_i^t$ or simply $\g_i$ instead of $\g_i(\x^t, \xiv_i^t)$ when there is no ambiguity.
    Then, the update rule in \Algn can be written as
    \begin{equation*}
        \x^{t+1} = \x^t - \gamma \sumin \w^{t+1}_i  \g_{i}(\x^{t}),
    \end{equation*}
    where $\w^{t+1}$ is an approximate solution of 
    \begin{eqnarray*}
        \min_{\w \Delta_1^n} \f\rbr*{\x^t - \gamma \sumin \w_i \g_{i}(\x^{t})}
    \end{eqnarray*}
    that satisfies
    \begin{eqnarray*}
        \E\sbr*{\f\rbr*{\x^{t+1}} | \x^t, \xiv^t} - \min_{\w} \f\rbr*{\x^t - \gamma \sumin \w_i \g_{i}(\x^{t})} \le \delta.
    \end{eqnarray*}
    By definition of the minimum, we have 
    \begin{eqnarray*}
        \lefteqn{\min_{\w\in \Delta^\n_1} \f\rbr*{\x^t - \gamma \sumin \w_i  \g_{i}(\x^{t})} \le \f\rbr*{\x^t - \frac{\gamma}{\gn} \sum_{i \in \cG}  \g_{i}(\x^{t})}}
        \\        &\oset{\eqref{eq:lipschitzness}}{\le}& \f\rbr*{\x^{t}} - \frac{\gamma}{\gn} \inp*{\nabla\f\rbr*{\x^{t}}}{\sum_{i \in \cG}  \g_{i}(\x^{t})} + \frac{L\gamma^2}{2}\norm*{\frac{1}{\gn}\sum_{i \in \cG}  \g_{i}(\x^{t})}^{2}
        \\
         &\le& \f\rbr*{\x^{t}} - \frac{\gamma}{\gn} \inp*{\nabla\f\rbr*{\x^{t}}}{\sum_{i \in \cG}  \g_{i}(\x^{t})} + \gamma^2 L\norm*{\nabla\f\rbr*{\x^{t}} - \frac{1}{\gn}\sum_{i \in \cG}  \g_{i}(\x^{t})}^{2} + \gamma^2 L\norm*{\nabla\f\rbr*{\x^{t}}}^{2}.
    \end{eqnarray*}
    The last two inequalities imply 
\begin{eqnarray}
        \lefteqn{\E\sbr*{\f\rbr*{\x^{t+1}} | \x^t, \xiv^t} } \notag
        \\
        &\le& \f\rbr*{\x^{t}} 
        - \frac{\gamma}{\gn} 
        \inp*{\nabla\f\rbr*{\x^{t}}}{\sum_{i \in \cG}  \g_{i}(\x^{t})} 
        + \gamma^2 L
        \norm*{
        \nabla\f\rbr*{\x^{t}} - \frac{\sum_{i \in \cG}  \g_{i}(\x^{t})}{\gn}
        }^{2} + \gamma^2 L
        \norm*{\nabla\f\rbr*{\x^{t}}}^{2} + \delta.
        \nonumber
    \end{eqnarray}
    Taking the full expectation we get
    \begin{eqnarray}
        \E[f(x^{t+1})] &\leq& \E[f(x^t)] - \gamma(1 - \gamma L)\E\left[\|\nabla f(x^t)\|^2\right] + \gamma^2L\E\left[\left\|\nabla\f\rbr*{\x^{t}} - \frac{\sum_{i \in \cG}  \g_{i}(\x^{t})}{\gn}\right\|^2\right] + \delta \notag\\
        &\overset{\gamma \leq \frac{1}{2L}}{\leq}&\E[f(x^t)] - \frac{\gamma}{2}\E\left[\|\nabla f(x^t)\|^2\right] + \frac{\gamma^2L}{G^2}\sum\limits_{i\in \cG}\E\left[\left\|\nabla\f\rbr*{\x^{t}} - \g_{i}(\x^{t})\right\|^2\right] + \delta \notag\\
        &\overset{\eqref{eq:xi-var}}{\leq}& \E[f(x^t)] - \frac{\gamma}{2}\E\left[\|\nabla f(x^t)\|^2\right] + \frac{\gamma^2 L\sigma^2}{G} + \delta. \label{eq:suff_decr}
    \end{eqnarray}
    The above is equivalent to
    \begin{eqnarray*}
        \frac{\gamma}{2}\E \norm*{\nabla\f\rbr*{\x^{t}}}^{2}
        \le \E\f\rbr*{\x^{t}} -  \E\f\rbr*{\x^{t+1}} + \frac{\sigma^2 \gamma^2 L}{\gn} + \delta,
    \end{eqnarray*}
    which concludes the first part of the proof.
    
    Next, summing the inequality for $t \in \{0,1,\dots,T-1\}$
    leads to
    \begin{eqnarray*}
        \frac{1}{T}\sum_{t=0}^{T-1}\E \norm*{\nabla\f\rbr*{\x^{t}}}^{2}
        &\le& \frac{2\rbr*{\f\rbr*{\x^{0}} -  \E\f\rbr*{\x^{T}}}}{T\gamma} + \frac{2\sigma^2 \gamma L }{\gn} + \frac{2\delta }{\gamma}
        \\
        &\le& 2\rbr*{\f\rbr*{\x^{0}} - \f\rbr*{\x^{*}}} + \frac{2\sigma^2 \gamma L }{\gn} + \frac{2\delta }{\gamma}.
    \end{eqnarray*}
    
    Combining~\eqref{eq:suff_decr} with~\eqref{eq:pl} gives
    \begin{eqnarray*}
        \E[f(x^{t+1}) - f^*] &\leq& (1 - \gamma\mu)\E[f(x^t) - f^*] + \frac{\gamma^2 L\sigma^2}{G} + \delta.
    \end{eqnarray*}
    Unrolling the above recurrence, we obtain \eqref{eq:PL_result_appendix}.
\end{proof}

\newpage

\section{Additional Experiments}\label{appendix:extra_exps}

\subsection{Results without Additional Validation Dataset}

In this section, we provide experiments without an additional dataset. Instead, we use the target client's train dataset to approximately solve the problem in Line~\ref{lst:line:aux_problem}. The results are provided in Figures~\ref{wfig:cifar-noval-0.5}-\ref{wfig:cifar-noval-0.99} (image classification) and Figures~\ref{wfig:bert-noval-0.5}-\ref{wfig:bert-noval-0.99} (text classification). They show that \Algn's behavior with and without additional validation data is almost the same. Thus, these preliminary results give evidence that our method can be efficient in practice even when an extra validation dataset is unavailable. 

\begin{figure*}[t]
    \begin{minipage}[htp]{0.24\textwidth}
        \centering
        \includegraphics[width=1\linewidth]{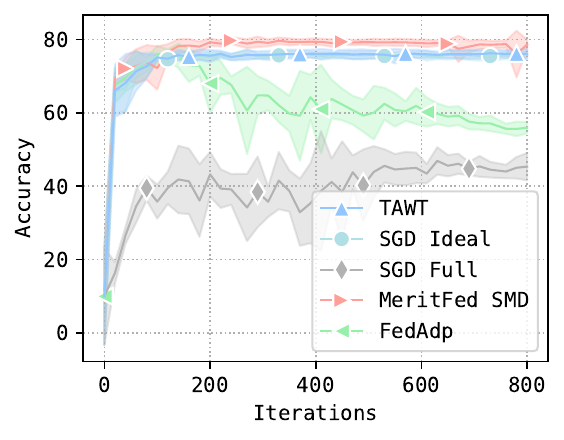}
        \includegraphics[width=1\linewidth]{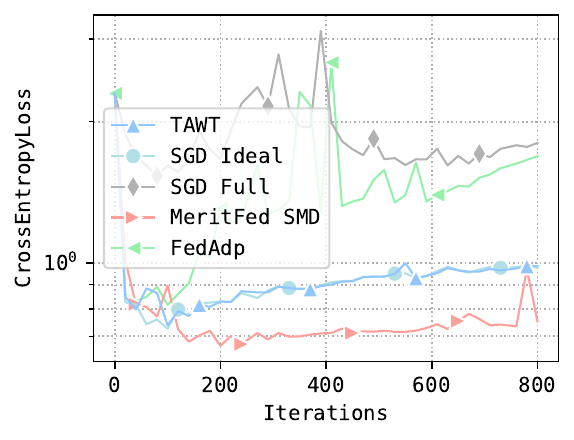}
        \includegraphics[width=1\linewidth]{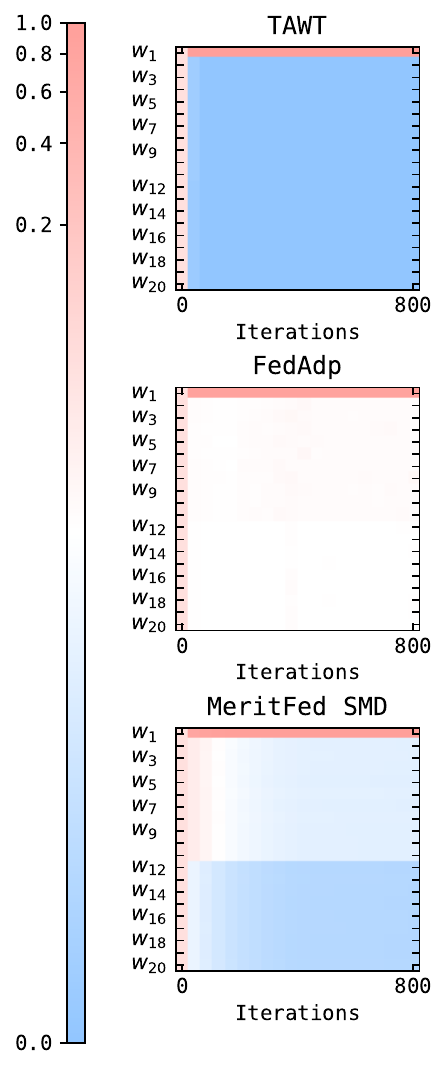}
        \caption{CIFAR10: $\alpha = 0.5$}
        \label{wfig:cifar-noval-0.5}
    \end{minipage}
\hfill
    \begin{minipage}[htp]{0.24\textwidth}
        \centering
        \includegraphics[width=1\linewidth]{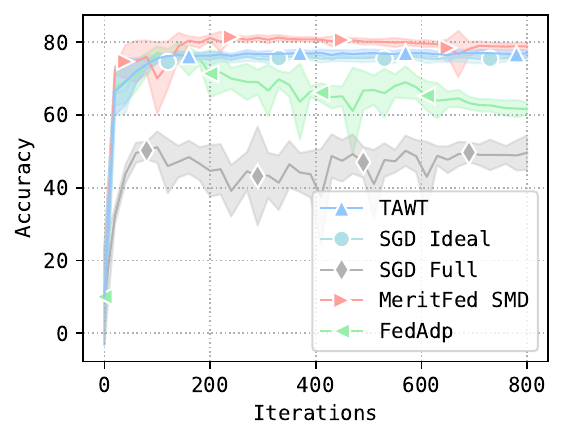}
        \includegraphics[width=1\linewidth]{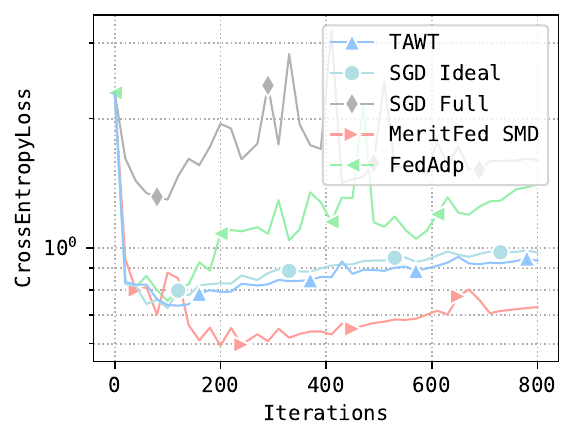}
        \includegraphics[width=1\linewidth]{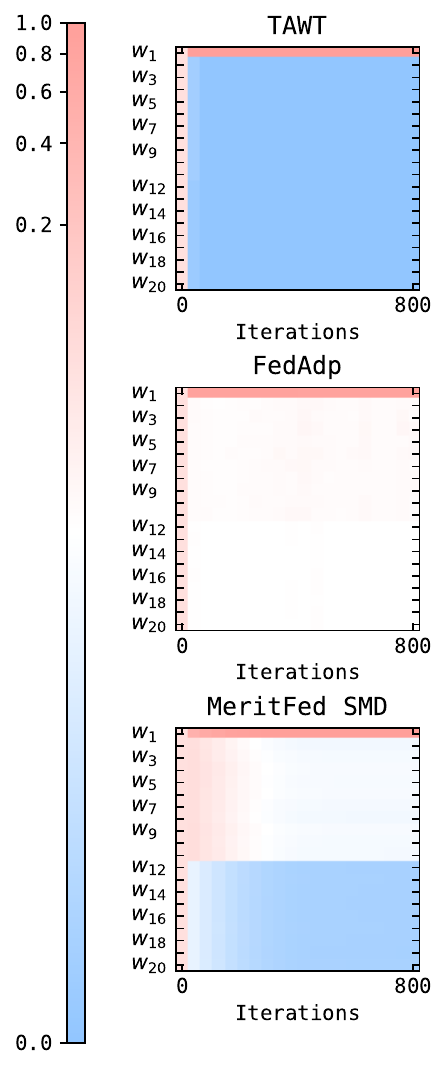}
        \caption{CIFAR10: $\alpha = 0.7$}
        \label{wfig:cifar-noval-0.7}
    \end{minipage}
\hfill
    \begin{minipage}[htp]{0.24\textwidth}
        \centering
        \includegraphics[width=1\linewidth]{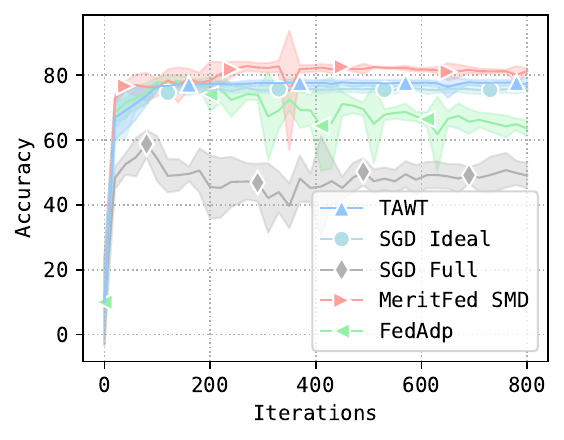}
        \includegraphics[width=1\linewidth]{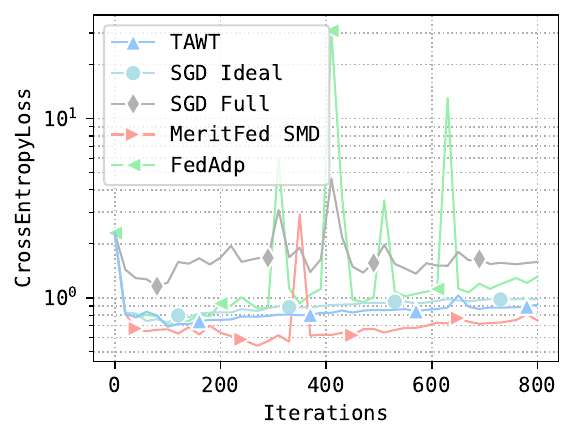}
        \includegraphics[width=1\linewidth]{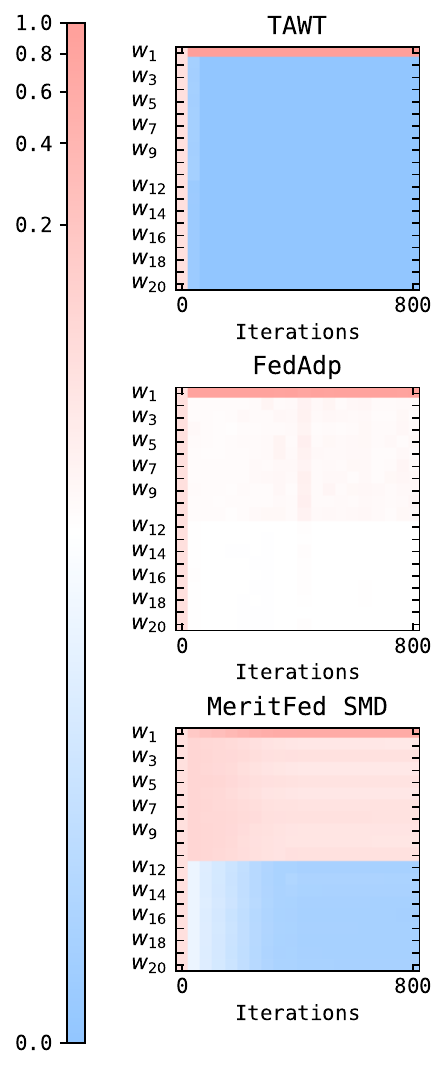}
        \caption{CIFAR10: $\alpha = 0.9$}
        \label{wfig:cifar-noval-0.9}
    \end{minipage}
\hfill
    \begin{minipage}[htp]{0.24\textwidth}
        \centering
        \includegraphics[width=1\linewidth]{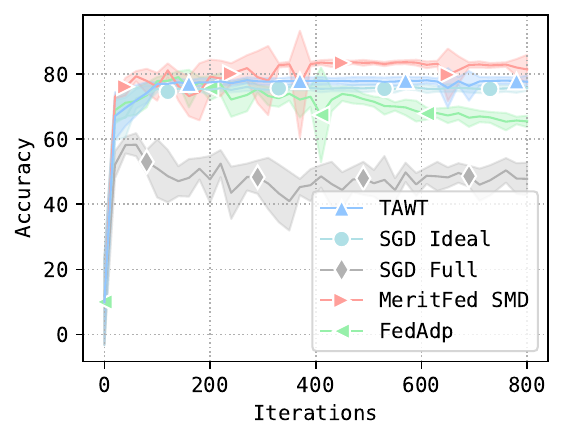}
        \includegraphics[width=1\linewidth]{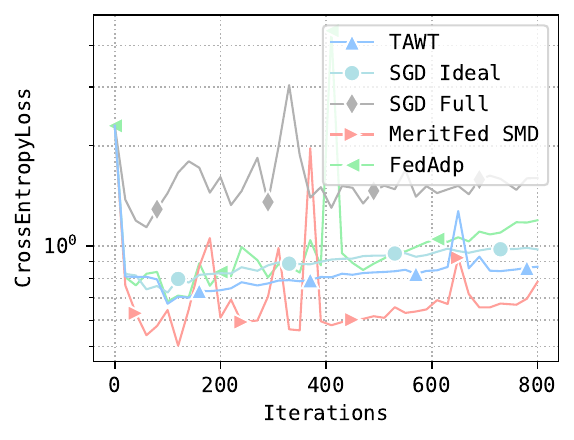}
        \includegraphics[width=1\linewidth]{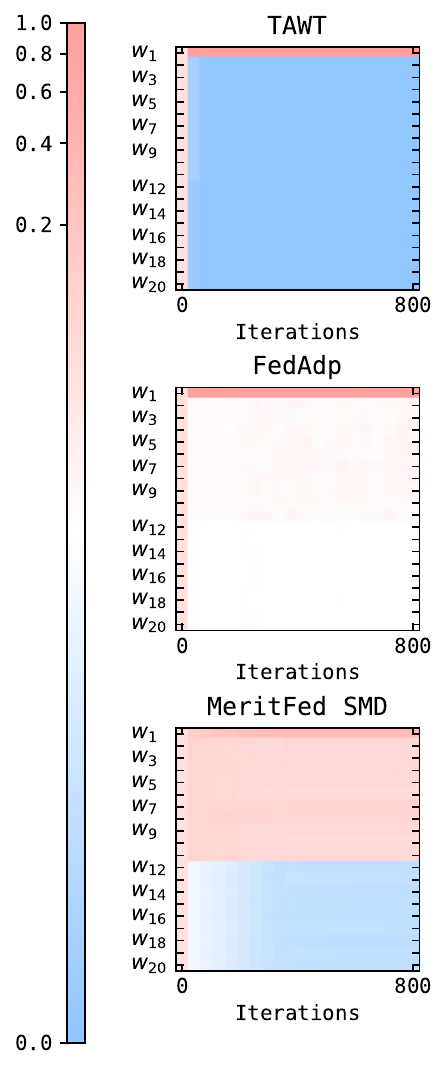}
        \caption{CIFAR10: $\alpha = 0.99$}
        \label{wfig:cifar-noval-0.99}
    \end{minipage}
\end{figure*}

\begin{figure*}[t]
    \begin{minipage}[htp]{0.24\textwidth}
        \centering
        \includegraphics[width=1\linewidth]{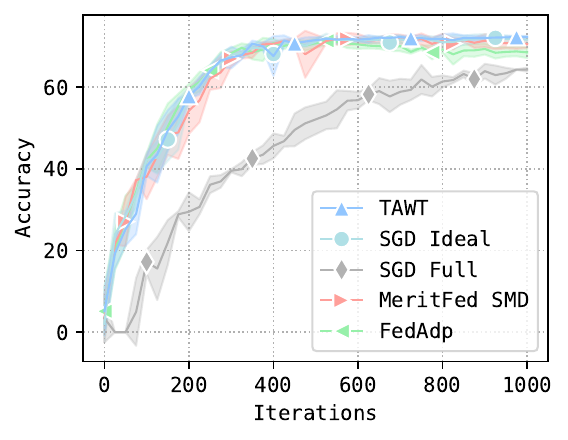}
        \includegraphics[width=1\linewidth]{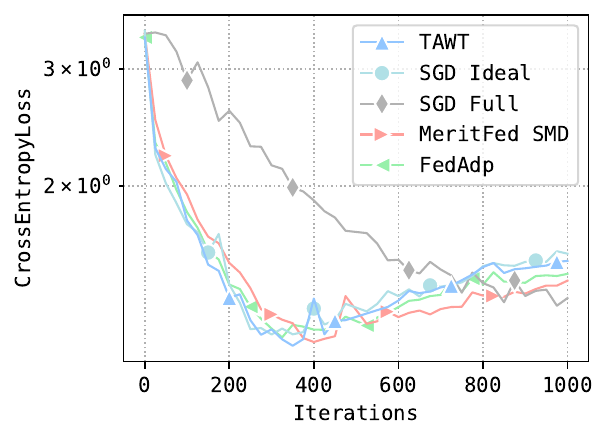}
        \includegraphics[width=1\linewidth]{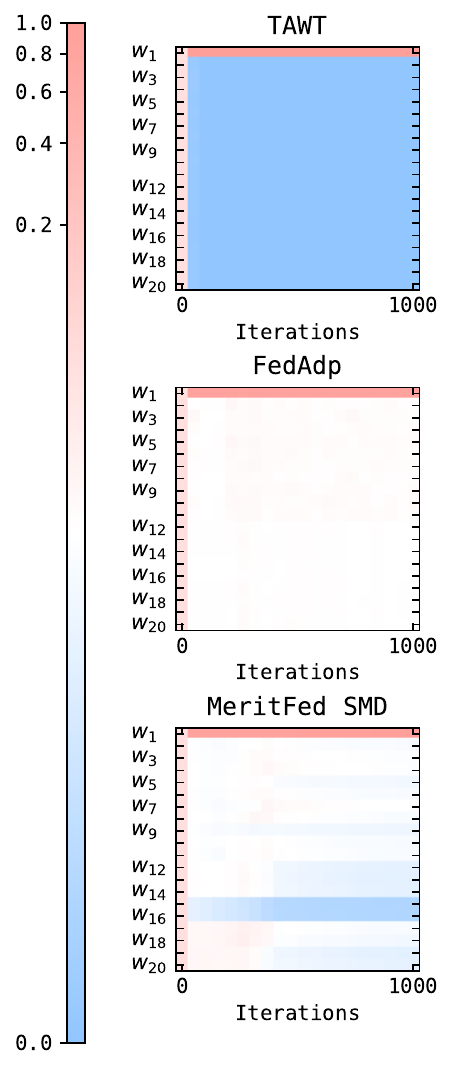}
        \caption{GoEmotions: $\alpha = 0.5$}
        \label{wfig:bert-noval-0.5}
    \end{minipage}
\hfill
    \begin{minipage}[htp]{0.24\textwidth}
        \centering
        \includegraphics[width=1\linewidth]{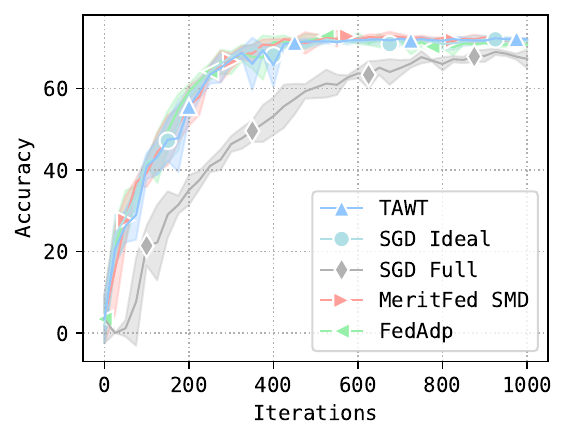}
        \includegraphics[width=1\linewidth]{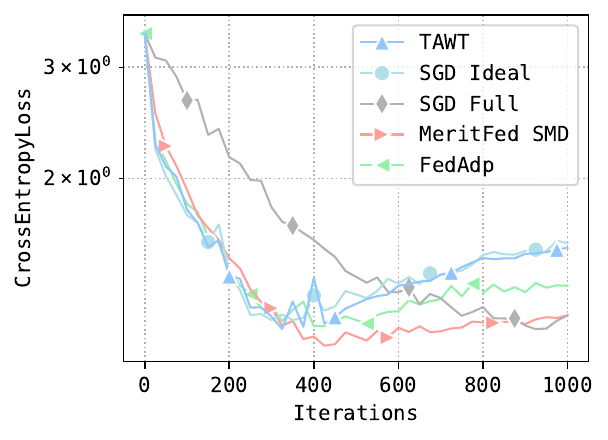}
        \includegraphics[width=1\linewidth]{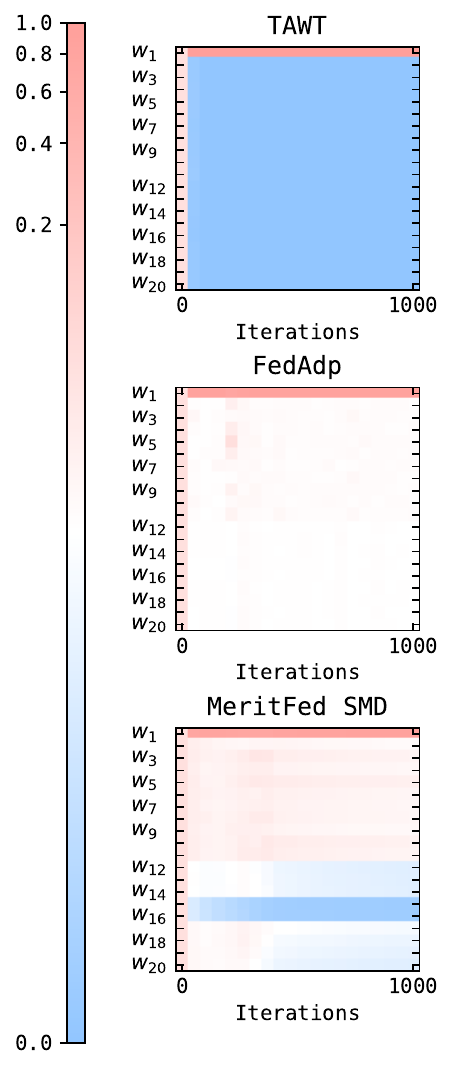}
        \caption{GoEmotions: $\alpha = 0.7$}
        \label{wfig:bert-noval-0.7}
    \end{minipage}
\hfill
    \begin{minipage}[htp]{0.24\textwidth}
        \centering
        \includegraphics[width=1\linewidth]{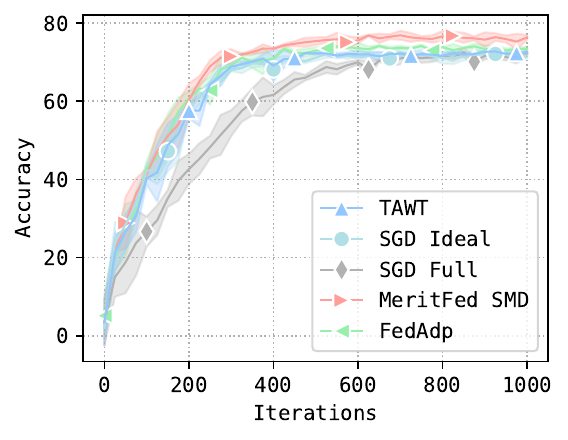}
        \includegraphics[width=1\linewidth]{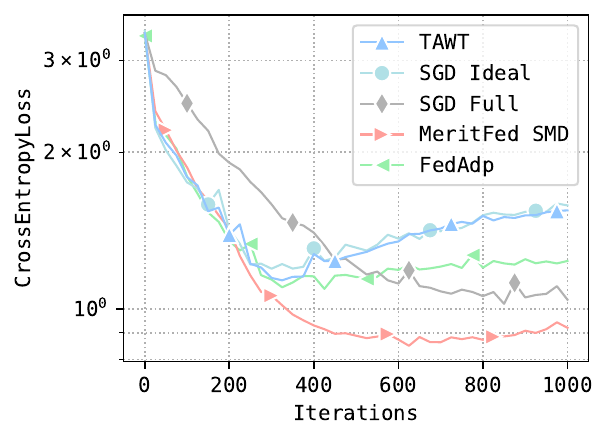}
        \includegraphics[width=1\linewidth]{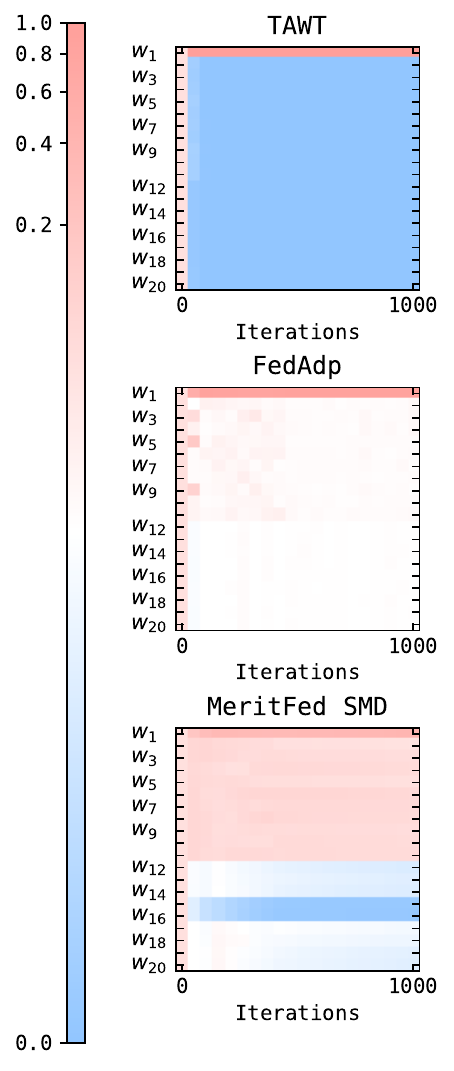}
        \caption{GoEmotions: $\alpha = 0.9$}
        \label{wfig:bert-noval-0.9}
    \end{minipage}
\hfill
    \begin{minipage}[htp]{0.24\textwidth}
        \centering
        \includegraphics[width=1\linewidth]{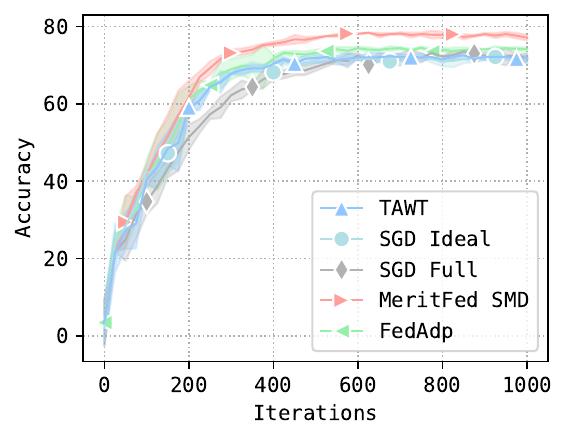}
        \includegraphics[width=1\linewidth]{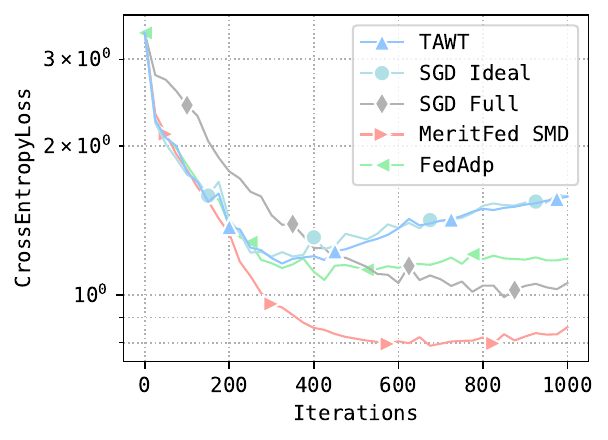}
        \includegraphics[width=1\linewidth]{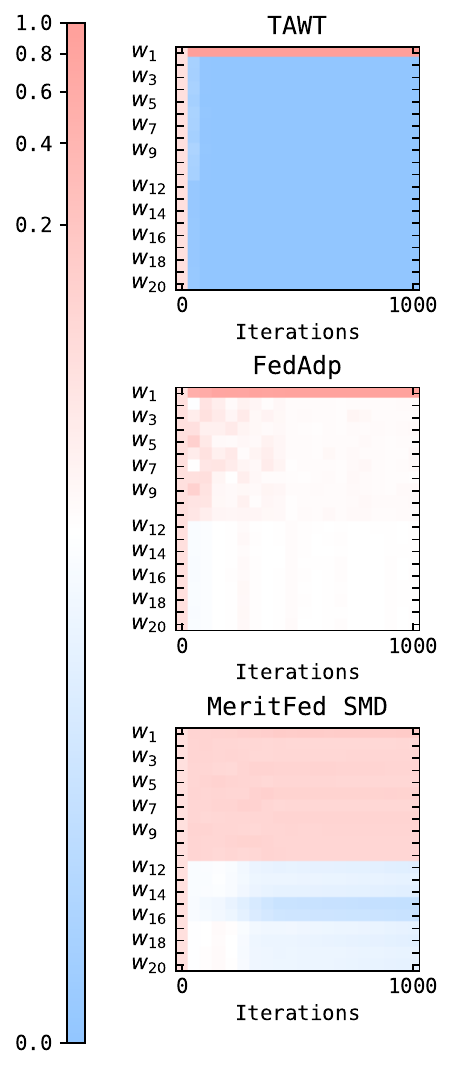}
        \caption{GoEmotions: $\alpha = 0.99$}
        \label{wfig:bert-noval-0.99}
    \end{minipage}
\end{figure*}

\subsection{Robustness against Byzantine Attacks}

\Algn is robust to Byzantine attacks since our proof of Theorem~\ref{thm:main_result} does not make any assumptions on the vectors received from the workers having different data distribution than the target client. This means that any worker $i \not\in \mathcal{G}$ can send arbitrary vectors at each iteration, and \Algn will still be able to converge. Moreover, \Algn can tolerate Byzantine attacks even if Byzantine workers form a majority, e.g., the method converges even if all clients are Byzantine except for the target one. 

To test the Byzantine robustness of our method on the mean estimation problem, we chose the total number of peers equal to $55$ with the $50$ clients being malicious. Malicious clients know the target distribution of the first $5$ client and use it for performing IPM (with parameter $\varepsilon_{\text{IPM}} = 0.1$) \citep{xie2019fall}  and ALIE (with parameter $z_{\text{ALIE}} = 100$) \citep{baruch2019little} attacks. We also consider the Bit Flipping\footnote{Byzantine workers compute stochastic gradients $g_i^k$ and send $-g_i^k$ to the server.} (BF) and the Random Noise\footnote{Byzantine workers compute stochastic gradients $g_i^k$ and send $g_i^k + \sigma\xi_i^k$ to the server, where $\xi_i^k \sim \cN(0, \mI)$ and $\sigma = 1$.} (RN) attacks. The following choice of parameters is used: each client has $1000$ samples from the corresponding distribution. The dimension of the problem is $d = 10$, learning rate $= 0.01$, number of steps for Mirror Descent $= 10$, learning rate for Mirror Descent $= 3.5$.

The results are presented in Figures~\ref{wfig:cifar-byz-0.5}-\ref{wfig:cifar-byz-0.99}. As expected, \algname{SGD Full} does not converge under the considered attacks, and \algname{SGD Ideal} shows the best results since, by design, it averages only with non-Byzantine workers. \algname{FedAdp} has poor performance under ALIE attack and is quite unstable under RN attack. As in other experiments, \algname{TAWT} is very biased towards the target client, which helps \algname{TAWT} to tolerate Byzantine attacks, but it does not take extra advantage of averaging with clients having the same distribution. Finally, \algname{MeritFed} consistently shows comparable results to \algname{SGD Ideal}.

\begin{figure}[H]
    \begin{minipage}[htp]{0.24\textwidth}
        \includegraphics[width=1\linewidth]{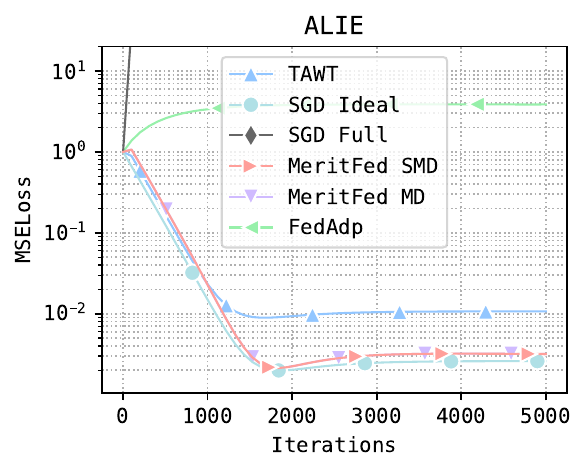}
        \includegraphics[width=1\linewidth]{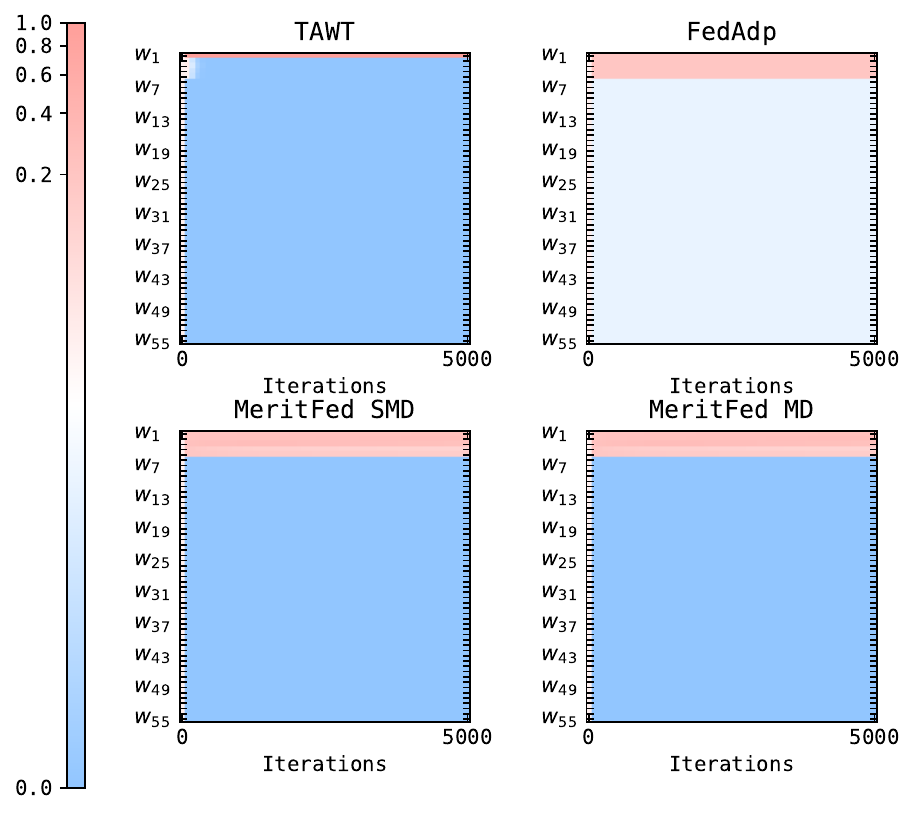}
        \caption{ALIE}
        \label{wfig:cifar-byz-0.5}
    \end{minipage}
\hfill
    \begin{minipage}[htp]{0.24\textwidth}
        \centering
        \includegraphics[width=1\linewidth]{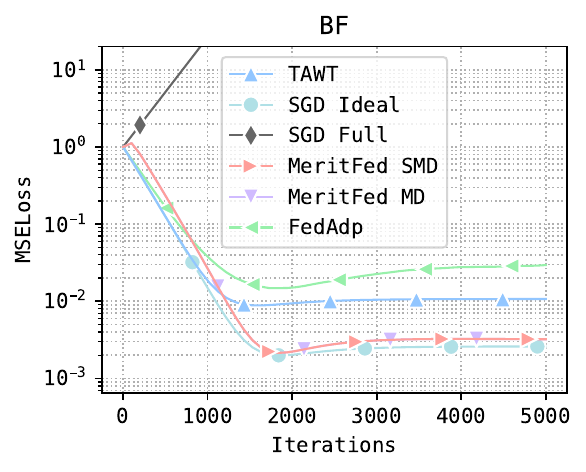}
        \includegraphics[width=1\linewidth]{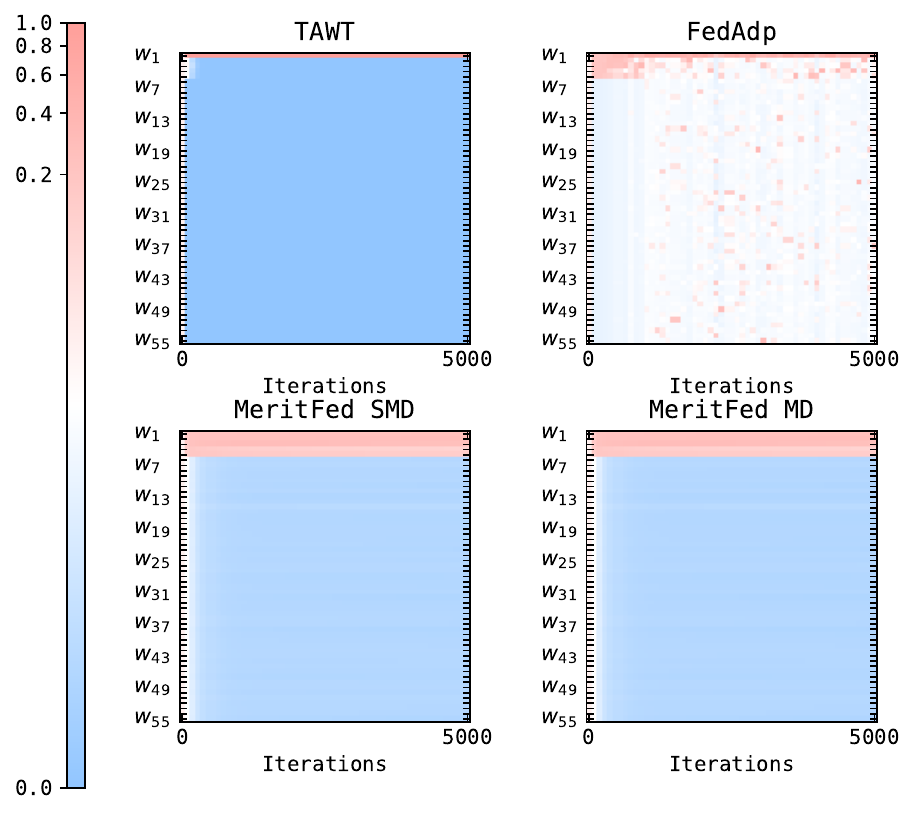}
        \caption{BF}
        \label{wfig:cifar-byz-0.7}
    \end{minipage}
\hfill
    \begin{minipage}[htp]{0.24\textwidth}
        \centering
        \includegraphics[width=1\linewidth]{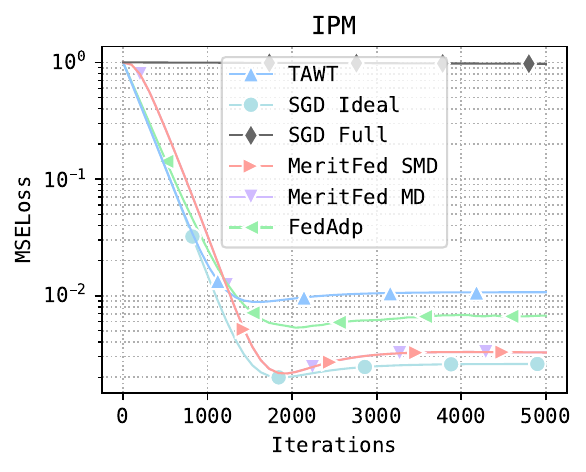}
        \includegraphics[width=1\linewidth]{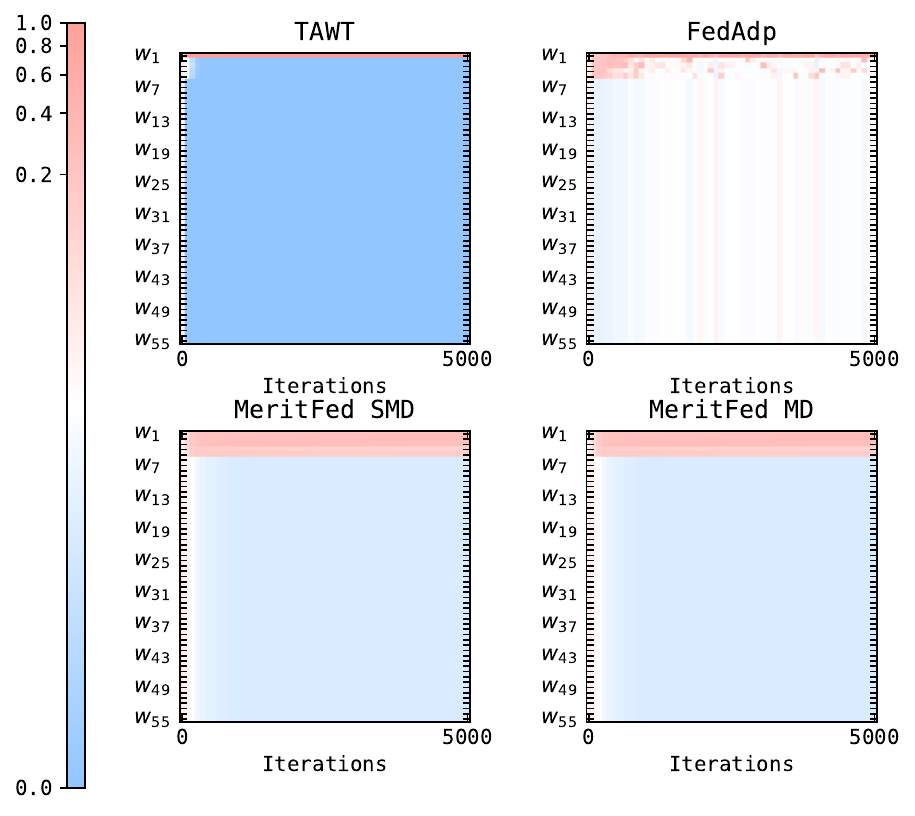}
        \caption{IPM}
        \label{wfig:cifar-byz-0.9}
    \end{minipage}
\hfill
    \centering
    \begin{minipage}[htp]{0.24\textwidth}
        \centering
        \includegraphics[width=1\linewidth]{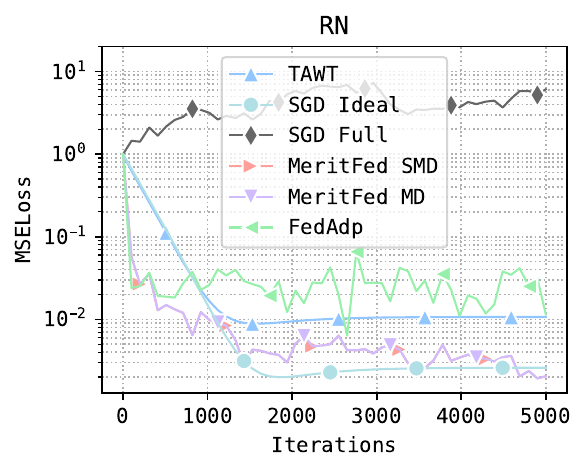}
        \includegraphics[width=1\linewidth]{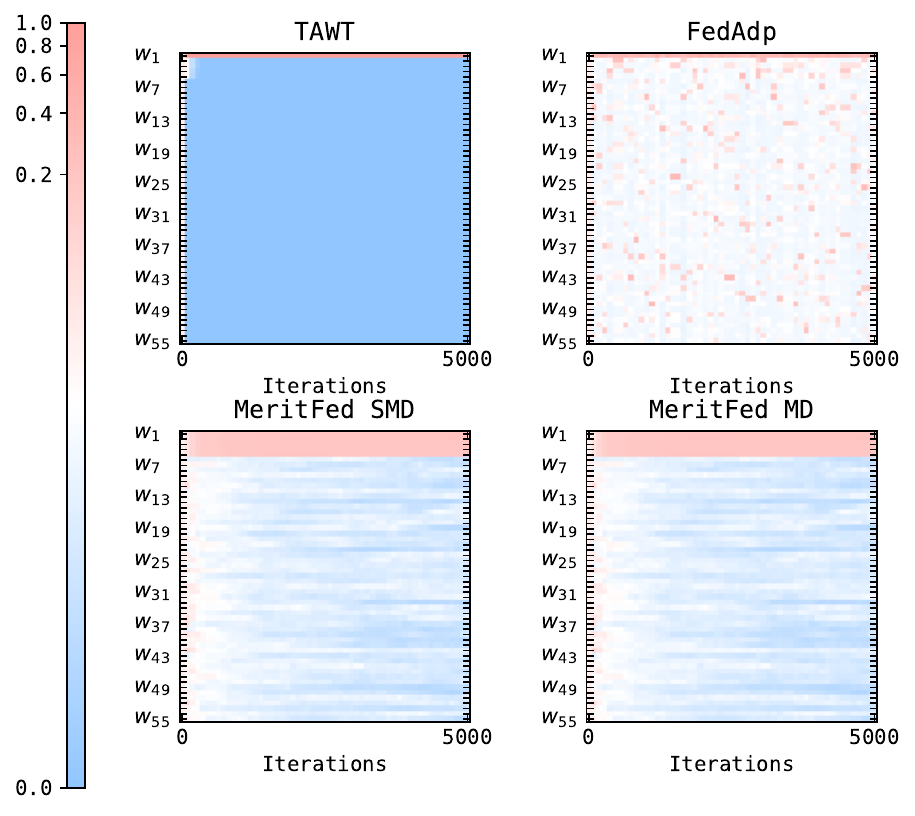}
        \caption{RN}
        \label{wfig:cifar-byz-0.99}
    \end{minipage}
\end{figure}

\subsection{ResNet18+CIFAR10: 40 workers}
In the mean estimation problem, we generate the data and can control the number of workers. Therefore, for this problem we have many clients participating in the training. 

However, for the other two tasks, datasets are fixed. Therefore, we limited the number of workers to 20 to have enough data on each client (given the splitting strategy) without repetition. That is, each data sample (image or tokens) from the original datasets belongs to no more than 1 client. Therefore, to run experiments with more workers we either need to have more data or allow repetitions in data on the clients. 

In the additional experiments, we have 40 clients where the new 20 clients are just copies of the first 20 clients. The experimental setup follows the same data partitioning idea as presented in the paper and deals with for values of heterogeneity values across clients $\alpha$. 
For \algname{MeritFed} each worker calculates stochastic gradient using a batch size of $75$; then the server uses Mirror Descent (or its stochastic version) with a batch-size of $90$ (in case of stochastic version) and a learning rate of $0.1$ to update weights of aggregation, and then performs a model parameters update with a learning rate of $0.01$.


The results presented on Figures~\ref{fig:40cifar-val-0.5}-\ref{fig:40cifar-val-0.99}. Overall, the conclusions are consistent with what we have in the experiment with 20 workers, further supporting the scalability of \Algn.

\begin{figure}[H]
    \begin{minipage}[htp]{0.24\textwidth}
        \centering
        \includegraphics[width=1\linewidth]{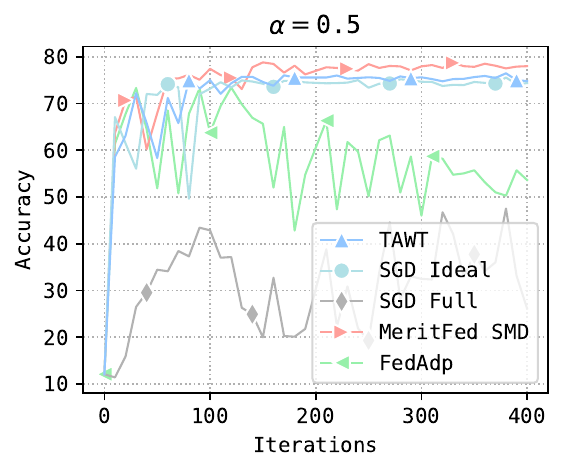}
        \includegraphics[width=1\linewidth]{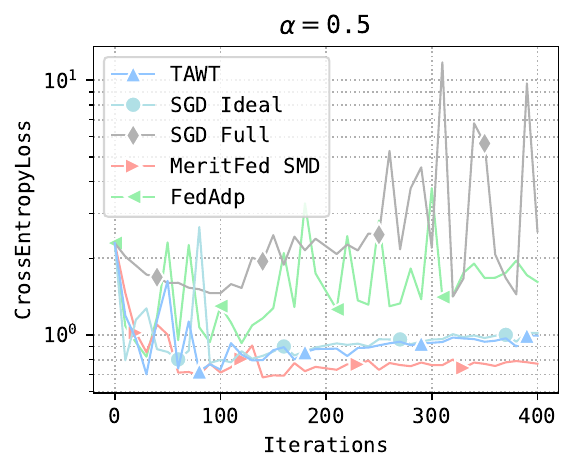}
        \includegraphics[width=1\linewidth]{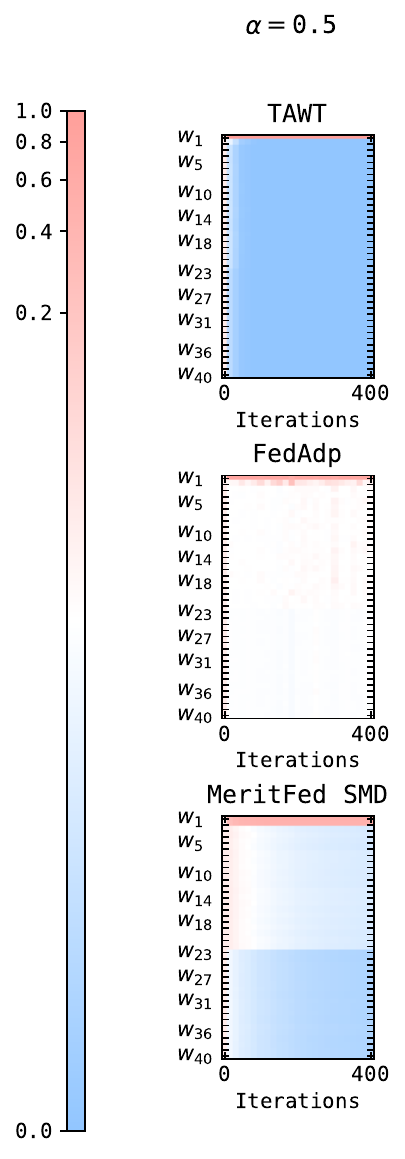}
        \vspace{-0.8cm}
        \caption{CIFAR10: $\alpha = 0.5$}
        \label{fig:40cifar-val-0.5}
    \end{minipage}
\hfill
    \begin{minipage}[htp]{0.24\textwidth}
        \centering
        \includegraphics[width=1\linewidth]{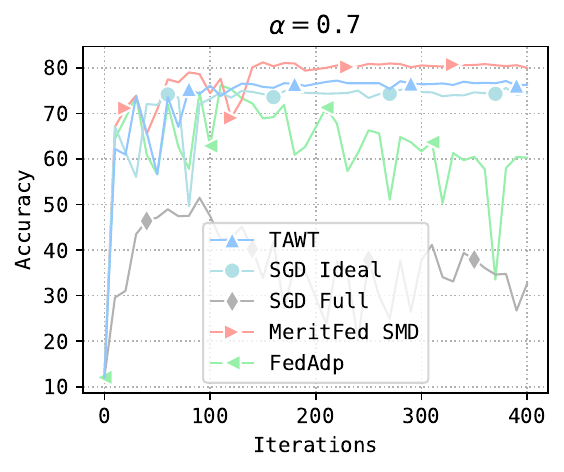}
        \includegraphics[width=1\linewidth]{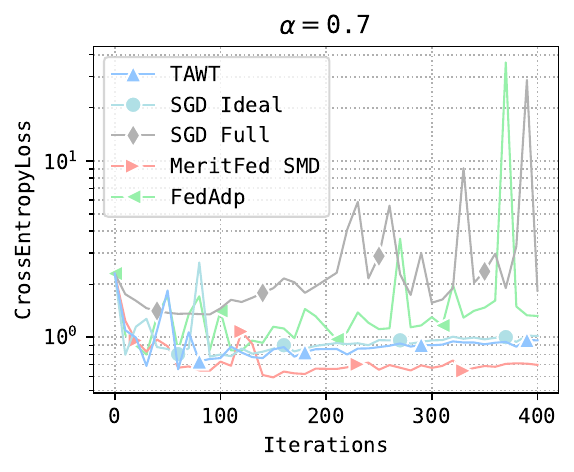}
        \includegraphics[width=1\linewidth]{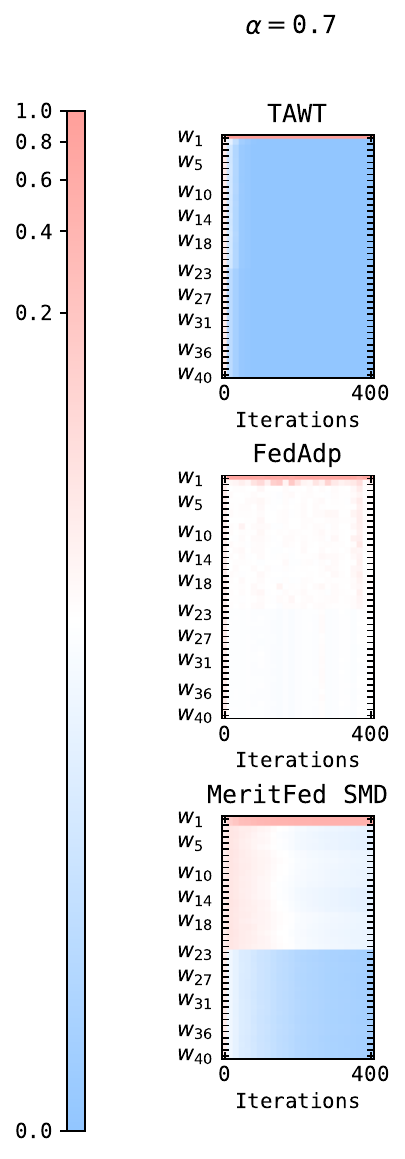}
        \vspace{-0.8cm}
        \caption{CIFAR10: $\alpha = 0.7$}
        \label{fig:40cifar-val-0.7}
    \end{minipage}
\hfill
    \begin{minipage}[htp]{0.24\textwidth}
        \centering
        \includegraphics[width=1\linewidth]{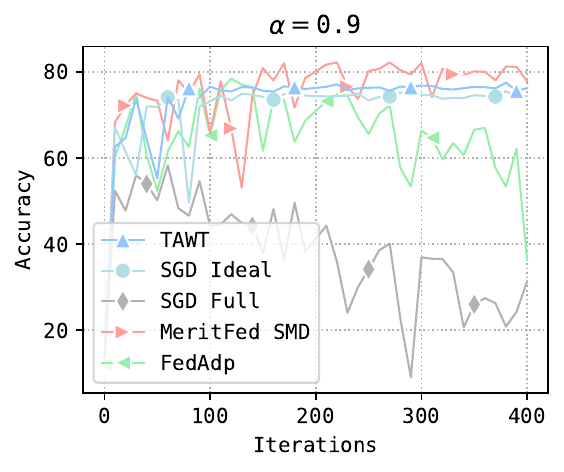}
        \includegraphics[width=1\linewidth]{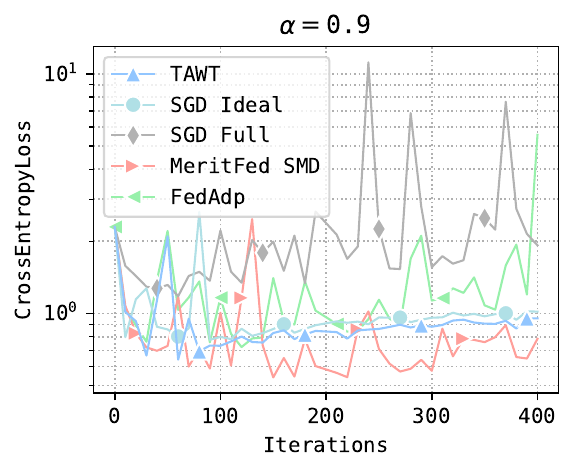}
        \includegraphics[width=1\linewidth]{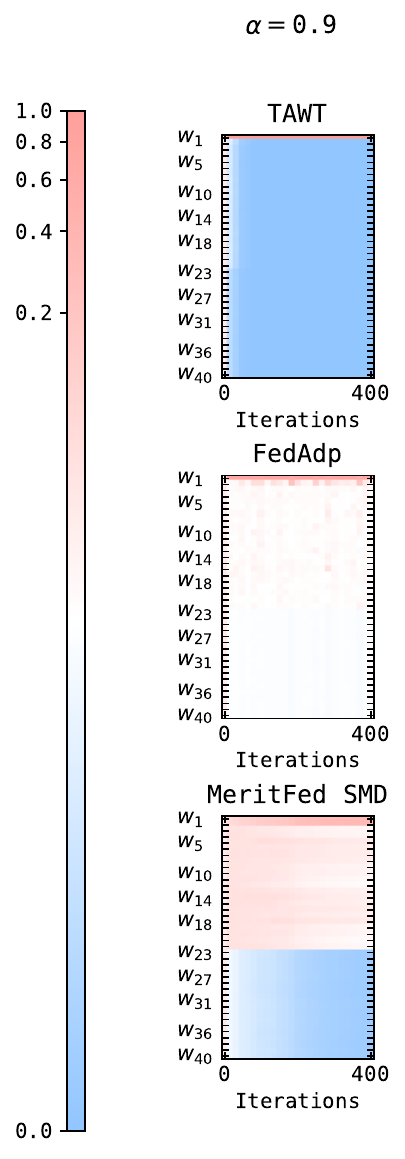}
        \vspace{-0.8cm}
        \caption{CIFAR10: $\alpha = 0.9$}
        \label{fig:40cifar-val-0.9}
    \end{minipage}
\hfill
    \begin{minipage}[htp]{0.24\textwidth}
        \centering
        \includegraphics[width=1\linewidth]{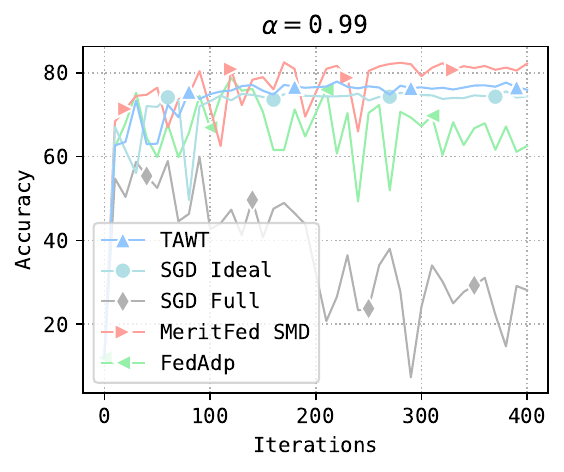}
        \includegraphics[width=1\linewidth]{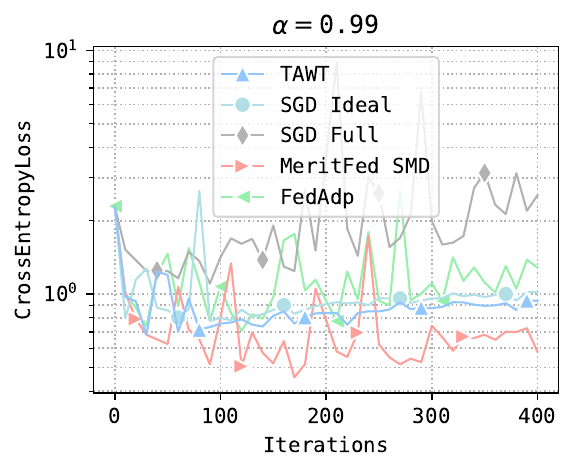}
        \includegraphics[width=1\linewidth]{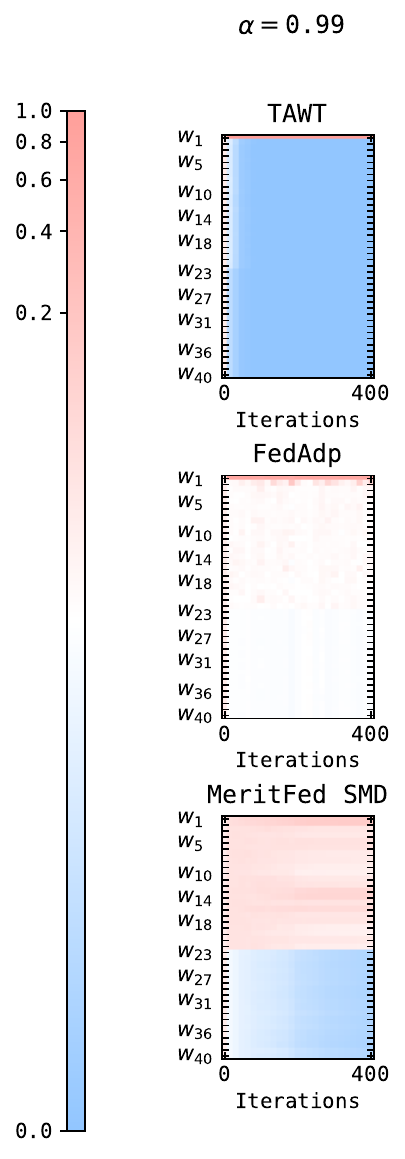}
        \vspace{-0.8cm}
        \caption{CIFAR10: $\alpha = 0.99$.}
        \label{fig:40cifar-val-0.99}
    \end{minipage}
    
\end{figure}


\newpage
\section*{NeurIPS Paper Checklist}

\begin{enumerate}

\item {\bf Claims}
    \item[] Question: Do the main claims made in the abstract and introduction accurately reflect the paper's contributions and scope?
    \item[] Answer: \answerYes{} 
    \item[] Justification: the new method are introduced in Section~\ref{sec:meritfed}, main convergence results are provided in Section~\ref{sec:convergence}, and the numerical results are provided in Section~\ref{sec:numerical_results}. 
    \item[] Guidelines:
    \begin{itemize}
        \item The answer NA means that the abstract and introduction do not include the claims made in the paper.
        \item The abstract and/or introduction should clearly state the claims made, including the contributions made in the paper and important assumptions and limitations. A No or NA answer to this question will not be perceived well by the reviewers. 
        \item The claims made should match theoretical and experimental results, and reflect how much the results can be expected to generalize to other settings. 
        \item It is fine to include aspirational goals as motivation as long as it is clear that these goals are not attained by the paper. 
    \end{itemize}

\item {\bf Limitations}
    \item[] Question: Does the paper discuss the limitations of the work performed by the authors?
    \item[] Answer: \answerYes{}
    \item[] Justification: see second paragraph of Section~\ref{sec:conclusion}
    \item[] Guidelines:
    \begin{itemize}
        \item The answer NA means that the paper has no limitation while the answer No means that the paper has limitations, but those are not discussed in the paper. 
        \item The authors are encouraged to create a separate "Limitations" section in their paper.
        \item The paper should point out any strong assumptions and how robust the results are to violations of these assumptions (e.g., independence assumptions, noiseless settings, model well-specification, asymptotic approximations only holding locally). The authors should reflect on how these assumptions might be violated in practice and what the implications would be.
        \item The authors should reflect on the scope of the claims made, e.g., if the approach was only tested on a few datasets or with a few runs. In general, empirical results often depend on implicit assumptions, which should be articulated.
        \item The authors should reflect on the factors that influence the performance of the approach. For example, a facial recognition algorithm may perform poorly when image resolution is low or images are taken in low lighting. Or a speech-to-text system might not be used reliably to provide closed captions for online lectures because it fails to handle technical jargon.
        \item The authors should discuss the computational efficiency of the proposed algorithms and how they scale with dataset size.
        \item If applicable, the authors should discuss possible limitations of their approach to address problems of privacy and fairness.
        \item While the authors might fear that complete honesty about limitations might be used by reviewers as grounds for rejection, a worse outcome might be that reviewers discover limitations that aren't acknowledged in the paper. The authors should use their best judgment and recognize that individual actions in favor of transparency play an important role in developing norms that preserve the integrity of the community. Reviewers will be specifically instructed to not penalize honesty concerning limitations.
    \end{itemize}

\item {\bf Theory Assumptions and Proofs}
    \item[] Question: For each theoretical result, does the paper provide the full set of assumptions and a complete (and correct) proof?
    \item[] Answer: \answerYes{} 
    \item[] Justification: see Section~\ref{sec:convergence} and the Appendix.
    \item[] Guidelines:
    \begin{itemize}
        \item The answer NA means that the paper does not include theoretical results. 
        \item All the theorems, formulas, and proofs in the paper should be numbered and cross-referenced.
        \item All assumptions should be clearly stated or referenced in the statement of any theorems.
        \item The proofs can either appear in the main paper or the supplemental material, but if they appear in the supplemental material, the authors are encouraged to provide a short proof sketch to provide intuition. 
        \item Inversely, any informal proof provided in the core of the paper should be complemented by formal proofs provided in appendix or supplemental material.
        \item Theorems and Lemmas that the proof relies upon should be properly referenced. 
    \end{itemize}

    \item {\bf Experimental Result Reproducibility}
    \item[] Question: Does the paper fully disclose all the information needed to reproduce the main experimental results of the paper to the extent that it affects the main claims and/or conclusions of the paper (regardless of whether the code and data are provided or not)?
    \item[] Answer: \answerYes{} 
    \item[] Justification: see Section~\ref{sec:numerical_results}
    \item[] Guidelines:
    \begin{itemize}
        \item The answer NA means that the paper does not include experiments.
        \item If the paper includes experiments, a No answer to this question will not be perceived well by the reviewers: Making the paper reproducible is important, regardless of whether the code and data are provided or not.
        \item If the contribution is a dataset and/or model, the authors should describe the steps taken to make their results reproducible or verifiable. 
        \item Depending on the contribution, reproducibility can be accomplished in various ways. For example, if the contribution is a novel architecture, describing the architecture fully might suffice, or if the contribution is a specific model and empirical evaluation, it may be necessary to either make it possible for others to replicate the model with the same dataset, or provide access to the model. In general. releasing code and data is often one good way to accomplish this, but reproducibility can also be provided via detailed instructions for how to replicate the results, access to a hosted model (e.g., in the case of a large language model), releasing of a model checkpoint, or other means that are appropriate to the research performed.
        \item While NeurIPS does not require releasing code, the conference does require all submissions to provide some reasonable avenue for reproducibility, which may depend on the nature of the contribution. For example
        \begin{enumerate}
            \item If the contribution is primarily a new algorithm, the paper should make it clear how to reproduce that algorithm.
            \item If the contribution is primarily a new model architecture, the paper should describe the architecture clearly and fully.
            \item If the contribution is a new model (e.g., a large language model), then there should either be a way to access this model for reproducing the results or a way to reproduce the model (e.g., with an open-source dataset or instructions for how to construct the dataset).
            \item We recognize that reproducibility may be tricky in some cases, in which case authors are welcome to describe the particular way they provide for reproducibility. In the case of closed-source models, it may be that access to the model is limited in some way (e.g., to registered users), but it should be possible for other researchers to have some path to reproducing or verifying the results.
        \end{enumerate}
    \end{itemize}

\item {\bf Open access to data and code}
    \item[] Question: Does the paper provide open access to the data and code, with sufficient instructions to faithfully reproduce the main experimental results, as described in supplemental material?
    \item[] Answer: \answerYes{} 
    \item[] Justification: see Section~\ref{sec:numerical_results}
    \item[] Guidelines:
    \begin{itemize}
        \item The answer NA means that paper does not include experiments requiring code.
        \item Please see the NeurIPS code and data submission guidelines (\url{https://nips.cc/public/guides/CodeSubmissionPolicy}) for more details.
        \item While we encourage the release of code and data, we understand that this might not be possible, so “No” is an acceptable answer. Papers cannot be rejected simply for not including code, unless this is central to the contribution (e.g., for a new open-source benchmark).
        \item The instructions should contain the exact command and environment needed to run to reproduce the results. See the NeurIPS code and data submission guidelines (\url{https://nips.cc/public/guides/CodeSubmissionPolicy}) for more details.
        \item The authors should provide instructions on data access and preparation, including how to access the raw data, preprocessed data, intermediate data, and generated data, etc.
        \item The authors should provide scripts to reproduce all experimental results for the new proposed method and baselines. If only a subset of experiments are reproducible, they should state which ones are omitted from the script and why.
        \item At submission time, to preserve anonymity, the authors should release anonymized versions (if applicable).
        \item Providing as much information as possible in supplemental material (appended to the paper) is recommended, but including URLs to data and code is permitted.
    \end{itemize}

\item {\bf Experimental Setting/Details}
    \item[] Question: Does the paper specify all the training and test details (e.g., data splits, hyperparameters, how they were chosen, type of optimizer, etc.) necessary to understand the results?
    \item[] Answer: \answerYes{} 
    \item[] Justification: see Section~\ref{sec:numerical_results}
    \item[] Guidelines:
    \begin{itemize}
        \item The answer NA means that the paper does not include experiments.
        \item The experimental setting should be presented in the core of the paper to a level of detail that is necessary to appreciate the results and make sense of them.
        \item The full details can be provided either with the code, in appendix, or as supplemental material.
    \end{itemize}

\item {\bf Experiment Statistical Significance}
    \item[] Question: Does the paper report error bars suitably and correctly defined or other appropriate information about the statistical significance of the experiments?
    \item[] Answer: \answerNA{} 
    \item[] Justification: the results are consistent for different runs
    \item[] Guidelines:
    \begin{itemize}
        \item The answer NA means that the paper does not include experiments.
        \item The authors should answer "Yes" if the results are accompanied by error bars, confidence intervals, or statistical significance tests, at least for the experiments that support the main claims of the paper.
        \item The factors of variability that the error bars are capturing should be clearly stated (for example, train/test split, initialization, random drawing of some parameter, or overall run with given experimental conditions).
        \item The method for calculating the error bars should be explained (closed form formula, call to a library function, bootstrap, etc.)
        \item The assumptions made should be given (e.g., Normally distributed errors).
        \item It should be clear whether the error bar is the standard deviation or the standard error of the mean.
        \item It is OK to report 1-sigma error bars, but one should state it. The authors should preferably report a 2-sigma error bar than state that they have a 96\% CI, if the hypothesis of Normality of errors is not verified.
        \item For asymmetric distributions, the authors should be careful not to show in tables or figures symmetric error bars that would yield results that are out of range (e.g. negative error rates).
        \item If error bars are reported in tables or plots, The authors should explain in the text how they were calculated and reference the corresponding figures or tables in the text.
    \end{itemize}

\item {\bf Experiments Compute Resources}
    \item[] Question: For each experiment, does the paper provide sufficient information on the computer resources (type of compute workers, memory, time of execution) needed to reproduce the experiments?
    \item[] Answer: \answerYes{} 
    \item[] Justification: all the details mentioned in Section \ref{sec:numerical_results}
    \item[] Guidelines:
    \begin{itemize}
        \item The answer NA means that the paper does not include experiments.
        \item The paper should indicate the type of compute workers CPU or GPU, internal cluster, or cloud provider, including relevant memory and storage.
        \item The paper should provide the amount of compute required for each of the individual experimental runs as well as estimate the total compute. 
        \item The paper should disclose whether the full research project required more compute than the experiments reported in the paper (e.g., preliminary or failed experiments that didn't make it into the paper). 
    \end{itemize}
    
\item {\bf Code Of Ethics}
    \item[] Question: Does the research conducted in the paper conform, in every respect, with the NeurIPS Code of Ethics \url{https://neurips.cc/public/EthicsGuidelines}?
    \item[] Answer: \answerYes{} 
    \item[] Justification: the paper follows NeurIPS Code of Ethics.
    \item[] Guidelines:
    \begin{itemize}
        \item The answer NA means that the authors have not reviewed the NeurIPS Code of Ethics.
        \item If the authors answer No, they should explain the special circumstances that require a deviation from the Code of Ethics.
        \item The authors should make sure to preserve anonymity (e.g., if there is a special consideration due to laws or regulations in their jurisdiction).
    \end{itemize}

\item {\bf Broader Impacts}
    \item[] Question: Does the paper discuss both potential positive societal impacts and negative societal impacts of the work performed?
    \item[] Answer: \answerNA{} 
    \item[] Justification: the paper is mostly theoretical and does not have a direct societal impact.
    \item[] Guidelines:
    \begin{itemize}
        \item The answer NA means that there is no societal impact of the work performed.
        \item If the authors answer NA or No, they should explain why their work has no societal impact or why the paper does not address societal impact.
        \item Examples of negative societal impacts include potential malicious or unintended uses (e.g., disinformation, generating fake profiles, surveillance), fairness considerations (e.g., deployment of technologies that could make decisions that unfairly impact specific groups), privacy considerations, and security considerations.
        \item The conference expects that many papers will be foundational research and not tied to particular applications, let alone deployments. However, if there is a direct path to any negative applications, the authors should point it out. For example, it is legitimate to point out that an improvement in the quality of generative models could be used to generate deepfakes for disinformation. On the other hand, it is not needed to point out that a generic algorithm for optimizing neural networks could enable people to train models that generate Deepfakes faster.
        \item The authors should consider possible harms that could arise when the technology is being used as intended and functioning correctly, harms that could arise when the technology is being used as intended but gives incorrect results, and harms following from (intentional or unintentional) misuse of the technology.
        \item If there are negative societal impacts, the authors could also discuss possible mitigation strategies (e.g., gated release of models, providing defenses in addition to attacks, mechanisms for monitoring misuse, mechanisms to monitor how a system learns from feedback over time, improving the efficiency and accessibility of ML).
    \end{itemize}
    
\item {\bf Safeguards}
    \item[] Question: Does the paper describe safeguards that have been put in place for responsible release of data or models that have a high risk for misuse (e.g., pretrained language models, image generators, or scraped datasets)?
    \item[] Answer: \answerNA{} 
    \item[] Justification: we do not release data or models.
    \item[] Guidelines:
    \begin{itemize}
        \item The answer NA means that the paper poses no such risks.
        \item Released models that have a high risk for misuse or dual-use should be released with necessary safeguards to allow for controlled use of the model, for example by requiring that users adhere to usage guidelines or restrictions to access the model or implementing safety filters. 
        \item Datasets that have been scraped from the Internet could pose safety risks. The authors should describe how they avoided releasing unsafe images.
        \item We recognize that providing effective safeguards is challenging, and many papers do not require this, but we encourage authors to take this into account and make a best faith effort.
    \end{itemize}

\item {\bf Licenses for existing assets}
    \item[] Question: Are the creators or original owners of assets (e.g., code, data, models), used in the paper, properly credited and are the license and terms of use explicitly mentioned and properly respected?
    \item[] Answer: \answerYes{} 
    \item[] Justification: see Section~\ref{sec:numerical_results}.
    \item[] Guidelines:
    \begin{itemize}
        \item The answer NA means that the paper does not use existing assets.
        \item The authors should cite the original paper that produced the code package or dataset.
        \item The authors should state which version of the asset is used and, if possible, include a URL.
        \item The name of the license (e.g., CC-BY 4.0) should be included for each asset.
        \item For scraped data from a particular source (e.g., website), the copyright and terms of service of that source should be provided.
        \item If assets are released, the license, copyright information, and terms of use in the package should be provided. For popular datasets, \url{paperswithcode.com/datasets} has curated licenses for some datasets. Their licensing guide can help determine the license of a dataset.
        \item For existing datasets that are re-packaged, both the original license and the license of the derived asset (if it has changed) should be provided.
        \item If this information is not available online, the authors are encouraged to reach out to the asset's creators.
    \end{itemize}

\item {\bf New Assets}
    \item[] Question: Are new assets introduced in the paper well documented and is the documentation provided alongside the assets?
    \item[] Answer: \answerNA{} 
    \item[] Justification: not applicable.
    \item[] Guidelines:
    \begin{itemize}
        \item The answer NA means that the paper does not release new assets.
        \item Researchers should communicate the details of the dataset/code/model as part of their submissions via structured templates. This includes details about training, license, limitations, etc. 
        \item The paper should discuss whether and how consent was obtained from people whose asset is used.
        \item At submission time, remember to anonymize your assets (if applicable). You can either create an anonymized URL or include an anonymized zip file.
    \end{itemize}

\item {\bf Crowdsourcing and Research with Human Subjects}
    \item[] Question: For crowdsourcing experiments and research with human subjects, does the paper include the full text of instructions given to participants and screenshots, if applicable, as well as details about compensation (if any)? 
    \item[] Answer: \answerNA{} 
    \item[] Justification: not applicable.
    \item[] Guidelines:
    \begin{itemize}
        \item The answer NA means that the paper does not involve crowdsourcing nor research with human subjects.
        \item Including this information in the supplemental material is fine, but if the main contribution of the paper involves human subjects, then as much detail as possible should be included in the main paper. 
        \item According to the NeurIPS Code of Ethics, workers involved in data collection, curation, or other labor should be paid at least the minimum wage in the country of the data collector. 
    \end{itemize}

\item {\bf Institutional Review Board (IRB) Approvals or Equivalent for Research with Human Subjects}
    \item[] Question: Does the paper describe potential risks incurred by study participants, whether such risks were disclosed to the subjects, and whether Institutional Review Board (IRB) approvals (or an equivalent approval/review based on the requirements of your country or institution) were obtained?
    \item[] Answer: \answerNA{} 
    \item[] Justification: not applicable.
    \item[] Guidelines:
    \begin{itemize}
        \item The answer NA means that the paper does not involve crowdsourcing nor research with human subjects.
        \item Depending on the country in which research is conducted, IRB approval (or equivalent) may be required for any human subjects research. If you obtained IRB approval, you should clearly state this in the paper. 
        \item We recognize that the procedures for this may vary significantly between institutions and locations, and we expect authors to adhere to the NeurIPS Code of Ethics and the guidelines for their institution. 
        \item For initial submissions, do not include any information that would break anonymity (if applicable), such as the institution conducting the review.
    \end{itemize}

\end{enumerate}

\end{document}

%% file: defs.tex
\usepackage[utf8]{inputenc} 
\usepackage[T1]{fontenc}    

\usepackage{url}            
\usepackage{booktabs}       
\usepackage{amsfonts}       
\usepackage{nicefrac}       
\usepackage{microtype}      
\usepackage{bbm}            
\usepackage{bm}             

\usepackage{eqparbox}
\newcommand{\oset}[3][B]{\mathrel{\eqmakebox[#1]{$\overset{#2}{#3}$}}} 

\newcommand{\eqdef}{\stackrel{\text{def}}{=}}

\newcommand{\dist}{\mathop{\mathrm{dist}}\nolimits}

\newcommand{\E}{\mathbb{E}}

\newcommand{\R}{\mathbb{R}}

\newcommand{\cD}{{\cal D}}

\newcommand{\cG}{{\cal G}}

\newcommand{\cN}{{\cal N}}
\newcommand{\cO}{{\cal O}}

\newcommand{\cS}{{\cal S}}

\newcommand\swapifbranches[3]{#1{#3}{#2}}
\makeatletter
\MHInternalSyntaxOn
\patchcmd{\DeclarePairedDelimiter}{\@ifstar}{\swapifbranches\@ifstar}{}{}
\MHInternalSyntaxOff
\makeatother

\DeclarePairedDelimiterX{\inp}[2]{\langle}{\rangle}{#1, #2}
\DeclarePairedDelimiterX{\abs}[1]{\lvert}{\rvert}{#1}
\DeclarePairedDelimiterX{\roundup}[1]{\lceil}{\rceil}{#1}
\DeclarePairedDelimiterX{\norm}[1]{\lVert}{\rVert}{#1}
\DeclarePairedDelimiterX{\cbr}[1]{\{}{\}}{#1} 
\DeclarePairedDelimiterX{\rbr}[1]{(}{)}{#1} 
\DeclarePairedDelimiterX{\sbr}[1]{[}{]}{#1} %

\DeclareMathOperator*{\argmin}{arg\,min}

\def\mI{{\bm{I}}}

